\documentclass{article}

\usepackage{microtype}
\usepackage{thm-restate}
\usepackage{graphicx}
\usepackage{subcaption}
\usepackage{capt-of}
\usepackage{booktabs} 

\usepackage{hyperref}

\usepackage[accepted]{icml2024}

\usepackage{amsmath}
\usepackage{amssymb}
\usepackage{mathtools}
\usepackage{amsthm}

\usepackage[capitalize,noabbrev]{cleveref}

\theoremstyle{plain}

\theoremstyle{definition}

\theoremstyle{remark}

\usepackage[textsize=tiny]{todonotes}

\icmltitlerunning{On the Diminishing Returns of Width for Continual Learning}

\begin{document}

\twocolumn[
\icmltitle{On the Diminishing Returns of Width for Continual Learning}




\begin{icmlauthorlist}
\icmlauthor{Etash Guha}{comp,uw}
\icmlauthor{Vihan Lakshman}{sch}
\end{icmlauthorlist}

\icmlaffiliation{comp}{SambaNova Systems, Palo Alto, USA}
\icmlaffiliation{sch}{ThirdAI, Houston, USA}
\icmlaffiliation{uw}{University of Washington, Seattle, USA}

\icmlcorrespondingauthor{Etash Guha}{\texttt{etash.guha@sambanovasystems.com}}
\icmlcorrespondingauthor{Vihan Lakshman}{\texttt{vihan@thirdai.com}}

\icmlkeywords{Machine Learning, ICML}

\vskip 0.3in
]



\printAffiliationsAndNotice{\icmlEqualContribution} 

\begin{abstract}
While deep neural networks have demonstrated groundbreaking performance in various settings, these models often suffer from \emph{catastrophic forgetting} when trained on new tasks in sequence. Several works have empirically demonstrated that increasing the width of a neural network leads to a decrease in catastrophic forgetting but have yet to characterize the exact relationship between width and continual learning. We designed one of the first frameworks to analyze continuous learning theory and prove that width is directly related to forgetting in feed-forward networks (FFN). In particular, we demonstrate that increasing network widths to reduce forgetting yields diminishing returns. We empirically verify our claims at widths hitherto unexplored in prior studies where the diminishing returns are clearly observed as predicted by our theory.
\end{abstract}

\section{Introduction}

Deep Neural Networks (DNNs) have achieved breakthrough performance in numerous challenging computational tasks that serve as a proxy for intelligence \citep{lecun2015deep}. An essential and practical question is whether the same neural network can continuously learn over a series of tasks while maintaining performance. Practically, using a well-trained neural network to achieve similar quality over a series of tasks is essential for reducing expensive retraining and computational costs \citep{diethe2019continual} and for mimicking the human-like ability to continually update its knowledge \citep{hadsell2020embracing, kudithipudi2022biological, parisi2019continual}. In practice, DNNs exhibit \emph{catastrophic forgetting} when trained over a series of tasks, experiencing a sharp drop in performance on the previously learned tasks.

However, preventing catastrophic forgetting is theoretically and empirically tricky in many situations \citep{knoblauch2020optimal, kim2022theoretical}. Many empirical studies of Continual Learning (CL) have observed that a model's hidden dimension or width is positively correlated with the ability to continually learn \citep{mirzadeh2022architecture, Mirzadeh2022, ramasesh2021effect}. Moreover, catastrophic forgetting is more easily mitigated in the infinite-width or Neural Tangent Kernel regime \citep{bennani2020generalisation, doan2021theoretical, chizat2019lazy, geiger2020disentangling}. However, the exact relationship between width and continual learning still needs to be clarified in the more practical finite-width setting. Theoretically, explaining this relationship requires understanding the effects of training and retraining a model on different datasets, which is a complex and difficult-to-analyze process. Through pure experimentation, most works observe a roughly linear relationship between width and continual learning, such as in small FFNs \citep{Mirzadeh2022} or in large CNNs and ResNets \citep{ramasesh2021effect, mirzadeh2022architecture}, suggesting that increasing the scale of models is a simple method for improving continual learning.

In this paper, we explicitly investigate the relationship between width and continual learning. We have observed empirically and theoretically that \emph{simply increasing a model's width suffers diminishing returns in improving continual learning.} To rigorously analyze these phenomena, we establish one of the first theoretical frameworks for analyzing the continual learning error of Feed-Forward Networks, connecting fundamental properties like depth, number of tasks, sparsity, activation smoothness, and width to continual learning error. We use this framework to prove the connection between width and continual learning in Feed-Forward Networks of arbitrary depth and nonlinear activations more tightly. To circumvent the typical analytical difficulties of continual learning, we use the well-observed empirical observation that wider models move less from initialization during training. For completeness, we empirically verify this relationship on Feed-Forward Networks trained with either Stochastic Gradient Descent (SGD) or Adam \citep{kingma2015adam} optimizers\footnote{https://github.com/vihan-lakshman/diminishing-returns-wide-continual-learning}. This lazy training phenomenon has often been observed empirically in the literature \citep{zou2020gradient, nagarajan2019generalization, li2018learning, neyshabur2018towards, chizat2019lazy, ghorbani2019limitations}. Using this observation, we demonstrate that width acts as a functional regularizer, preventing models trained on subsequent tasks from being too functionally distant from previous models. Specifically, our new guarantees formalize this relationship between width and continual learning for finite-width models with nonlinear activations and variable depth, which do not exist in the literature for wide models. To our knowledge, this theoretical framework is one of the first to provide provable guarantees on the continual learning error of models.

Moreover, we empirically observe these diminishing returns. Some existing works have tested hidden dimensions up to $2048$ \citep{mirzadeh2020linear} where the diminishing returns are not apparent. In particular, we measure the continual learning capabilities of FFNs as the hidden dimension is increased to $2^{16}$, much larger than previously explored in the literature to our knowledge. With these new expansive experiments on standard CL benchmarks, we see this relationship between width and continual learning predicted by our theory. We also test the other correlations predicted by our framework, such as the connections between model depth, model sparsity, and the number of tasks on continual learning error. Our framework predicts that increasing model depth or the number of tasks will also increase continual learning error, matching the empirical observations in \cite{mirzadeh2022architecture}.
Moreover, our model predicts increasing the row-wise sparsity in a model decreases the continual learning error, such as used in \citep{serra2018overcoming}. Our experiments roughly corroborate all of these relationships over several real datasets and model shapes. To our knowledge, this is the first work to provably demonstrate the effects of the number of tasks and row-wise sparsity on continual learning error.  

Our results contribute to the literature examining the relationship between neural network architectures and continual learning performance. We provide one of the first theoretical frameworks for analyzing catastrophic forgetting in Feed-Forward Networks. While our theoretical framework does not perfectly capture all information about continual forgetting empirically, it is a valuable step in analyzing continual learning from a theoretical framework. As predicted by our theoretical framework, we demonstrate empirically that scaling width alone is insufficient for mitigating the effects of catastrophic forgetting, providing a more nuanced understanding of finite-width forgetting dynamics than results achieved in prior studies \citep{Mirzadeh2022, ramasesh2021effect}. We also prove that our theoretical framework predicts several other connections between model architecture and catastrophic forgetting.

\paragraph{Contributions} In summary, we make the following contributions in our work.
\begin{enumerate}
\item We develop one of the first theoretical frameworks to analyze catastrophic forgetting in Feed-Forward Networks. Our theoretical framework corroborates several empirical findings, such as the connection between depth, sparsity, and forgetting. Our theoretical framework also predicts that models incur diminishing returns in terms of continual learning capability as width is increased.
    \item Under this framework, we provably demonstrate that the training of nonlinear, variable depth feed-forward networks incurs continual learning error on the order of $\mathcal{O}\left(t W^{-\beta} \alpha^{\frac{1-2\beta}{2}}\right)$ where $t$ is the number of tasks the model has been trained on, $W$ is the width, $\alpha$ is the sparsity percentage, and $\beta$ is a data-dependent positive value. To our knowledge, this is the first work formalizing the connection between width, sparsity, number of tasks, and continual learning for nonlinear models of variable depth. 
    \item By testing at hidden dimensions not seen previously, we empirically see the diminishing returns of continual learning when increasing width. These experiments hold across many different width Feed Forward Networks on datasets such as Rotated MNIST and Fashion MNIST, Rotated SVHN, and the Rotated German Traffic Signs Benchmark (GTSRB). 
    \item We empirically confirm the predictions of our theoretical frameworks on the impact of depth, number of tasks, and row-wise sparsity on the continual learning error over the same host of datasets, demonstrating the power of our theoretical framework. 
\end{enumerate}
\section{Related Works}

\subsection{Continual Learning}
We review several relevant works in the Continual Learning literature. The original works discussing Continual Learning and Catastrophic Forgetting phenomenon are \citet{ring1997child}, \citet{mccloskey1989catastrophic} and \citet{thrun1995lifelong}. Perhaps most relevant to this work are \citet{mirzadeh2022architecture} and \citet{Mirzadeh2022}, which note the positive correlation between the width of models and continual learning. However, their experiments are limited to small widths, at a maximum of $2048$ hidden dimension, which is small relative to the current deep learning state of the art. Moreover, their analysis is limited, only proving the continual learning in the setting of two-layer linear networks without nonlinear activation. \citet{ramasesh2021effect} and \citet{yoon2018lifelong} discuss how scale broadly empirically affects continual learning but does not provide any theoretical analysis nor focus specifically on width.

\citet{bennani2020generalisation} and \citet{doan2021theoretical} discuss theoretical frameworks for continually learning models in the infinitely wide or NTK regime. Several works have used explicit functional regularization as a way to mitigate catastrophic forgetting \citep{lopez2017gradient, chaudhry2018efficient, farajtabar2020orthogonal, khan2021knowledge, Dhawan2023}. \citet{peng2023ideal} provides a robust theoretical framework, their Ideal Continual Learner framework, which attempts to develop a solid theoretical framework for understanding different continual learning methods. \citet{mirzadeh2020linear} discuss the geometric similarities between the different minima found in a continual learning regime

\subsection{Wide Networks}
Here, we review essential works relevant to understanding wide networks in the literature and different architectures. \citet{nguyen2020wide} discusses the effects of width and depth on the learned representation of the model. \citet{arora2019exact} discusses the empirical benefits of using the infinite-width Neural Tangent Kernel on different classification tasks. \citet{novak2022fast} provides a method to compute the infinite-width model with less memory and better efficiency. \citet{lee2019wide} discuss the training dynamics of wide neural networks under gradient descent. \citet{lu2017expressive} discuss the expressiveness of wide neural networks and how wide neural networks can express certain functions better than deep ones. \citet{jacot2018neural}, \citet{allen2019learning}, and \citet{pmlr-v97-du19c} capture the training dynamics of infinitely-wide neural networks under the Neural Tangent Kernel Regime

\section{Preliminary}

\subsection{Notation}
\label{sec:notation}
Let our model $\mathbf{M}_t$ yielded after training on the $t$th task be denoted as 
$$\mathbf{M}_t(x) = \mathbf{A}_{t, L}\phi_{L-1}(\mathbf{A}_{t, L-1}\phi_{L-2}(\dots \mathbf{A}_{t, 2}\phi_1(\mathbf{A}_{t,1}x)))\text{.}$$ 
Here, $x$ is an input of dimensionality $d_t$. $\phi_{i}$ is the activation function for the $i$th layer of $L_i$ Lipschitz-Smoothness. Moreover let $W$ be the width of $\mathbf{M}$ such that the input layer $\mathbf{A}_{t, 1} \in \mathbb{R}^{W \times d_t}$, last layer $\mathbf{A}_{t, L} \in \mathbb{R}^{K_t \times W}$, and all the middle layers $\mathbf{A}_{t, l} \in \mathbb{R}^{W \times W}$. Here, $K_t$ is the dimensionality of the output of the $t$th task. We will often index a matrix by a set of rows. For example, if $\mathcal{S}$ is a set of row indeces, $\mathbf{A}_{t, l}[\mathcal{S}]$ denotes a matrix in $\mathbb{R}^{|\mathcal{S}| \times W}$ that contains the $i$th row from $\mathbf{A}_{t, l}$ if $i \in \mathcal{S}$. Moreover, let $L$ be the model depth.

\subsection{Problem Setup}
\label{sec:problem_setup}
Here, we will formalize the problem setup of Continual Learning. Formally, say we have $T$ training datasets $\mathcal{D}_1, \dots, \mathcal{D}_T$. The goal of continual learning is to design a model $\mathbf{M}$ such that it performs well on all datasets $\mathcal{D}_t$ for $t \in [T]$. We will describe the $t$th task as a supervised learning classification task where the dataset for the $t$th task is $\mathcal{D}_t = (\mathcal{X}_t, \mathcal{Y}_t)$. Here, $\mathcal{X}_t$ contains $n_t$ datapoints of dimension $d_t$ and $\mathcal{Y}_t$ contains $n_t$ labels of dimension $K_t$. We wish to form $\mathbf{M}$ by sequentially training it on each dataset in increasing order from $\mathcal{D}_1$ to $\mathcal{D}_t$. We will call $\mathbf{M}_t$ the model outputted after training on the $t$th dataset. After retraining a model on the new dataset, we want the new model to remember its behavior on the previous dataset. Namely, we wish to reduce the continual learning error $\epsilon_{t,t^{\prime}}$ where $t \leq t^{\prime}$ and 
$\underset{x \in \mathcal{D}_t}{\max} \left \| \mathbf{M}_t(x) - \mathbf{M}_{t^{\prime}}(x)\right\|_2 \leq \epsilon_{t,t^{\prime}}\text{.}$ Here, after training a model for $t^{\prime}-t$ new datasets, we hope to ensure that the outputs between the original model on the $t$th dataset and the new model on the $t^{\prime}$th dataset are similar. Given $\mathbf{M}_t$ is trained to completion on dataset $\mathcal{D}_t$ and achieves low error, if our new model $\mathbf{M}_{t^{\prime}}$ is close to $\mathbf{M}_t$ on all inputs from $\mathcal{D}_t$, it will also perform well on $\mathcal{D}_t$. Therefore, the main question is the correlation between $W$, the width of the model $\mathbf{M}$, and the continual learning error $\epsilon_{t,t^{\prime}}$.

\subsection{Training Setup}
\label{sec:training_setup}
We will examine the setup where our model $\mathbf{M}$ is a Feed Forward Network with some nonlinear activation. We randomly initialize every layer in the entire model to train the model on the first task $\mathbf{M}_1$. Before training, as is often done in Continual Learning, we will choose a subset of rows to be active. For every row with probability $\alpha$, we will select that row to be active. Otherwise, it will be inactive. Only the active rows will be used for computation during training and inference. Inactive rows will not change during training. We will denote $\mathcal{A}_{t, l}$ as the set of rows active for task $t$ at layer $l$. We use this setup to capture the connections between row-wise sparsity and continual learning empirically observed in continual learning literature \citep{serra2018overcoming}. Setting $\alpha$ to $1$ will recover fully dense training. We then train using Adam or SGD till convergence. To train on a subsequent task, we replace the input and output layers to match the dimensionality of the new tasks. Moreover, we choose which rows will be active for this new task for the intermediate layers. We then retrain till convergence with Adam or SGD. We repeat this training procedure iteratively for every task. For inference on the $t$th task, we take the intermediate layers learned at the final task and replace the input and output layers with the input and output layers trained for the $t$th task to match the dimensionality of the data from the $t$th task. We only use the active rows during training for the $t$th task.

\section{Theoretical Analysis}
\begin{figure*}[t!]
    \centering
    
    \subfloat[MNIST]{\includegraphics[width=0.48\textwidth]{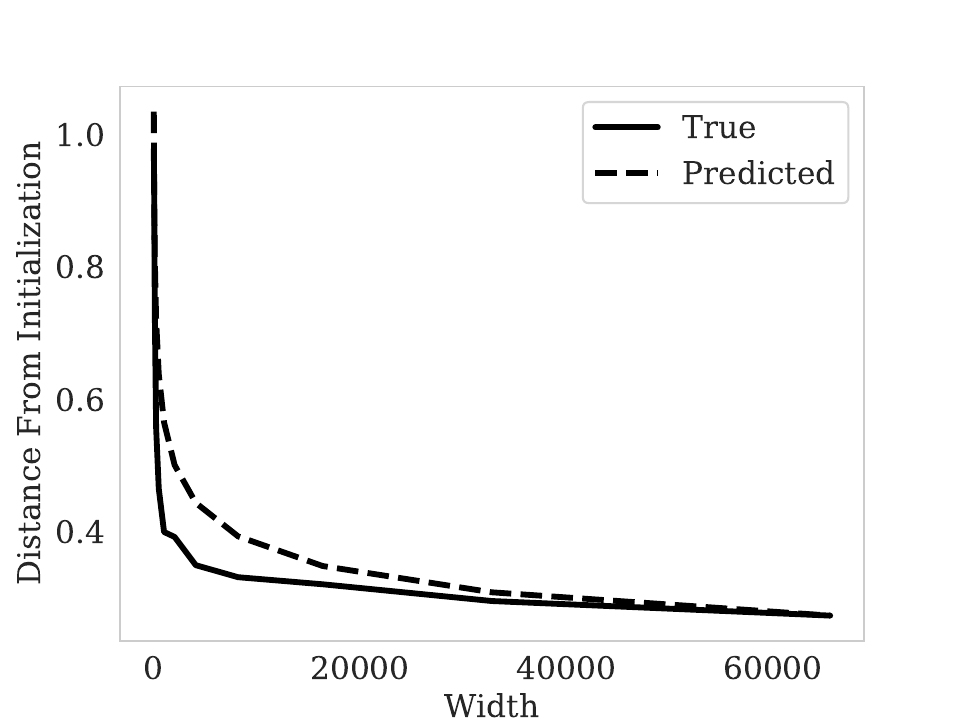}\label{fig:distmnist}}\hfill
    \subfloat[Fashion MNIST]{\includegraphics[width=0.48\textwidth]{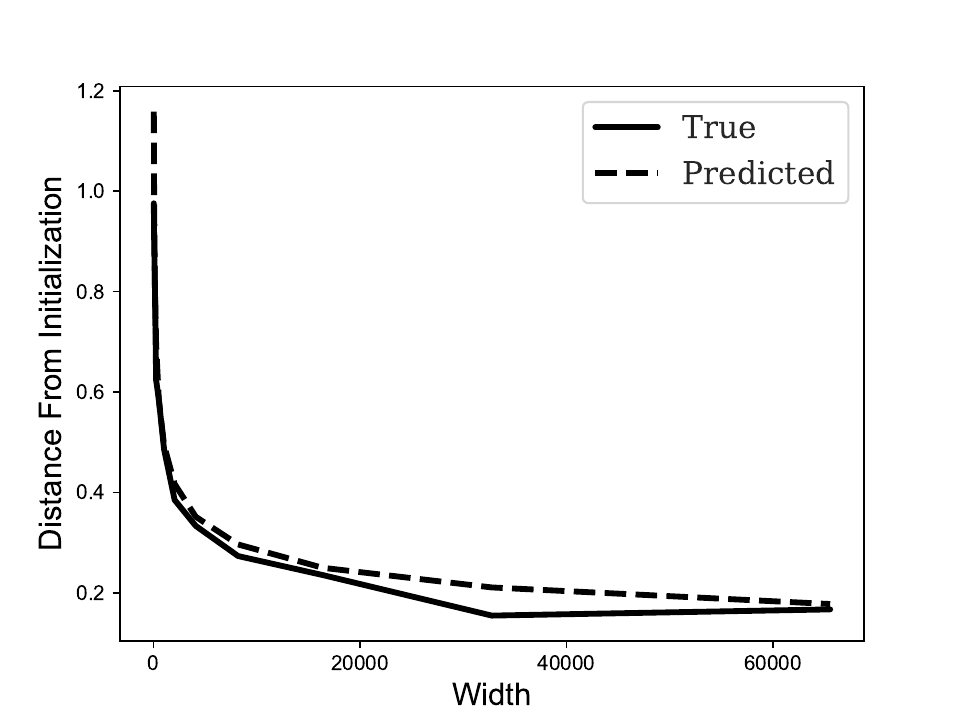}\label{fig:distcifar}}
    
    \caption{We plot the distance from initialization for both Rotated MNIST and Rotated Fashion MNIST experiments. We see that distance from initialization decreases slowly as the width is increased for both datasets. For the constants discussed in \Cref{ass:distance}, the best fitting constants are $\gamma = 0.013, \beta=0.311$ for Rotated MNIST and $\gamma = 2.5, \beta=0.12$ for Fashion MNIST. We plot the predicted relationship with such parameters from \Cref{ass:distance}.}

    \label{fig:distancefrominit}
\end{figure*}

We develop the theoretical connection between the width of the intermediate layers $W$ and the error of continual learning $\epsilon_{t, t^{\prime}}$ for different task indices $t$ and $t^{\prime}$. We begin by stating our final theorem and then provide a brief proof sketch. We mention a complete proof in \Cref{sec:theorem51}.

\subsection{Main Theorem}
Here, we present our full theorem, \Cref{lem:ourperturbation}
\begin{restatable}{theorem}{ourperturbation}
    \label{lem:ourperturbation}
     (Informal) Say we generate a series of models $\mathbf{M}_1, \dots, \mathbf{M}_T$ by training sequentially on datasets $\mathcal{D}_1, \dots, \mathcal{D}_T$ according to \Cref{sec:training_setup}. Let $\lambda_{i, j}^l = \frac{\|\mathbf{A}_{l, j}[\mathcal{A}_{l, i}]\|_2}{\|\mathbf{A}_{l, i}[\mathcal{A}_{l, i}]\|_2}$ denote the ratio of the spectral norms of the weights of different row indices for different tasks. Moreover, let $\bar{\lambda} = \underset{l \in [L], i, j \in [T]}{\max}\lambda_{i, j}^l$. For all input vectors from the $t$th dataset $\forall x \in \mathcal{D}_t$, the $\ell_2$ norm of the difference of the outputs from models $\mathbf{M}_t$ and $\mathbf{M}_{t^{\prime}}$ such that $t^{\prime} \geq t$ are upper bounded \footnote{We can reduce the dependence on weight norms by using noise stability properties. For more details, please see \Cref{sec:noisestable}.} by
    \begin{align}
        \mathbb{E}\bigg[\|&\mathbf{M}_t(x) - \mathbf{M}_{t^{\prime}}(x) \|_2 \bigg]=   \nonumber\\
        &\mathcal{O}\left((t^{\prime} - t)L2^L\bar{\lambda}\chi \left(\prod_{l=1}^L L_l\|\mathbf{A}_{t, l}\|_2 \right) \gamma W^{-\beta} \alpha^{\frac{1-2\beta}{2}}\right)\text{.}\nonumber
    \end{align}
    Here, $\chi$ denotes the maximum norm of the input in $\mathcal{D}_t$, i.e. $\chi = \underset{x \in \mathcal{D}_t}{\max} \|x\|_2$. Here, $\gamma, \beta$ are data-dependent positive real values. 
\end{restatable}
As the width of the layers increases, we see that the error $\epsilon_{t, t^{\prime}}$ is decreased on the order of $W^{-\beta}$. This formalization furthermore formalizes the connection between the width and the continual learning error $\epsilon_{t,t^{\prime}}$. This theorem demonstrates that the continual learning error will decrease slowly as the width increases. As side effects of our analysis, we also make explicit how this error scales over tasks. As the number of datasets trained between $\mathcal{D}_{t^{\prime}}$ and $\mathcal{D}_t$ increases, the continual learning error increases linearly. To our knowledge, this linear increase in error over tasks is unique to this analysis. Moreover, we have an exponential dependence on $L$, denoting that increasing depth decreases continual learning ability. \citet{mirzadeh2022architecture} also found that forgetting increased as depth increased. We reproduce these experiments on our datasets. Moreover, as the sparsity coefficient is decreased, meaning the chance of a row being active is decreased, the continual learning ability is increased as long as $\frac{1 - 2\beta}{2}$ is positive, which is often the case, in practice. Thus, this suggests a tradeoff between continual learning ability and accuracy, as increasing sparsity will increase continual learning ability but decrease the capacity of the model to learn on a given task. We note some dependence on the layers' matrix norms, which may correlate with the layers' width, but this is usually initialization dependent \citep{nagarajan2019generalization}. We can change the dependence on the weight norms using the noise stability property of neural networks as observed in \citet{pmlr-v80-arora18b} (for more details, please see \Cref{sec:noisestable}). Overall, this theorem makes the relationship between several model parameters, such as width and continual learning, concrete and makes the diminishing returns more explicit.

\subsection{Proof Sketch}
Here, we present a brief proof sketch of \Cref{lem:ourperturbation} and the intuition behind it. We will begin this proof by finding the continual learning error between subsequentially trained models $\mathbf{M}_{t}$ and $\mathbf{M}_{t+1}$, i.e. $\epsilon_{t, t+1}$, and then scale the analysis to work over more tasks $t^{\prime} \geq t+1$. To analyze $\epsilon_{t, t+1}$, we split the proof into three parts: (1) finding how many active rows are shared between layers in subsequential models at the same position, (2) finding how far these active rows can change during training, and (3) combining the two parts using perturbation analysis.

\subsubsection{Number of Shared Active Rows}

From \Cref{sec:training_setup}, each row in a layer $\mathbf{A}_{t,l}$ is active with probability $\alpha$ and inactive with probability $1 - \alpha$. Moreover, the activity of the rows are independent and identical random variables. Therefore, over the randomness of row selection, the expected number of shared active rows between two consecutive models is shown in \Cref{lem:intersection}.
 \begin{restatable}{lemma}{intersection}
\label{lem:intersection}
For any two sequential task indices $t$ and $t+1$ and layer $l$, the expected size of the intersection between the sets of active rows $\mathcal{A}_{t, l}$ and $\mathcal{A}_{t+1, l}$ is  $\mathbb{E}(|\mathcal{A}_{t, l} \cap \mathcal{A}_{t+1, l}|) =  \alpha^2 W\text{.}$
\end{restatable}
\subsubsection{Distance between active rows after training}
For rows in weight matrices in $\mathbf{M}_t$ and $\mathbf{M}_{t+1}$ that are active for both datasets $\mathcal{D}_t$ and $\mathcal{D}_{t+1}$, we need to bound their distance. Here, we use a critical empirical observation. As observed in \citet{zou2020gradient, nagarajan2019generalization, li2018learning, neyshabur2018towards, chizat2019lazy}, as the width of the neural network was increased, the distance from initialization was observed to decrease. Similar intuition on a "lazy training regime" was presented in \citet{Mirzadeh2022}. Still, they did not offer any formal analysis based on this observation. Intuitively, as the width of the neural network increases, the implicit regularization of standard gradient-based learning methods finds models closer to the initialization. While it is challenging to provide theoretical analyses of the connection between width, implicit regularization, and the distance from initialization save for restricted settings \citep{li2018learning}, empirically, this is a well-observed phenomenon. For the sake of our analysis, we take this as an assumption in \Cref{ass:distance} and build our analysis on top of it.
\begin{restatable}{assumption}{distance}
    \label{ass:distance}
    Let $t$ be any task index $t \in [T]$. After training on dataset $\mathcal{D}_{t + 1}$ with initialization $\mathbf{M}_{t}$ via gradient-based learning methods to generate $\mathbf{M}_{t+1}$, for all layers $l$, we have that $$\frac{\|\mathbf{A}_{l, t+1}[\mathcal{A}_{l, t+1}] - \mathbf{A}_{l, t}[\mathcal{A}_{l, t+1}]\|_F}{\|\mathbf{A}_{l, t}[\mathcal{A}_{l, t+1}]\|_2} \leq \gamma|\mathcal{A}_{l, t+1}|^{-\beta}$$ where $W$ is the width of the layers of $\mathbf{M}$ where $\gamma, \beta > 0$. Here, $\mathbf{A}_{l, t}[\mathcal{A}_{l, t+1}]$ and $\mathbf{A}_{l, t+1}[\mathcal{A}_{l, t+1}]$ denote matrices that only contain the active rows from the weights at the $l$th layer for the tasks $t$ and $t+1$ respectively. 
\end{restatable}
This assumption states that after training when indexed by the active rows, the matrix norm of the difference between the $l$th layer of $\mathbf{M}_t$ and the $l$th layer of $\mathbf{M}_{i+1}$ normalized by the initial matrix norm is bounded by some function which is inversely proportional to the width. This relative distance from initialization is well analyzed in the literature. For more details on these values in practice, see \Cref{sec:distfrominit}. We reiterate that this phenomenon has been observed across the literature, and we reproduce this property empirically as seen in \Cref{fig:distancefrominit}. Using this assumption, we can demonstrate an upper bound in the matrix norm of the difference of the weights at a layer for two sequentially trained models. 
\begin{restatable}{lemma}{perturbanytask}
    \label{lem:perturbanytask}
    Let $\lambda_{i, j}^l = \frac{\|\mathbf{A}_{l, j}[\mathcal{A}_{l, i}]\|_2}{\|\mathbf{A}_{l, i}[\mathcal{A}_{l, i}]\|_2}$ denote the ratio of the spectral norms of the weights of different row indices for different tasks. For any task $t$  and layer $l$, we have 
    $$\mathbb{E}\left[\frac{\left\|\mathbf{A}_{t, l}[\mathcal{A}_{t, l}] - \mathbf{A}_{t+1, l}[\mathcal{A}_{t, l}]\right \|_2}{\|\mathbf{A}_{t, l}[\mathcal{A}_{t, l}]\|_2} \right]\leq  \lambda_{t, t+1}^l\gamma W^{-\beta} \alpha^{\frac{1-2\beta}{2}} \text{.}$$ 
\end{restatable}

The matrix norm of the difference between two layers measures the functional distance of the two layers. Given this upper bound on the matrix norm, we can use simple perturbation analysis to calculate how much the difference between the two outputs of the two sequentially trained models accumulates throughout the layers during inference. 
\subsubsection{Width as a Functional Reguralizer}
We now use simple perturbation analysis to see how far $\mathbf{M}_t$ and $\mathbf{M}_{t^{\prime}}$ will be on input from $\mathcal{D}_t$. Given the Lipschitz-smoothness of the activation functions and a bound on the matrix norm of the weights of different layers from \Cref{lem:perturbanytask}, we can use Lemma 2 from \citet{neyshabur2018towards}, to bound the continual learning error. This analysis gives us a simple guarantee on the error $\epsilon_{t, t+1}$. Moreover, using simple triangle inequality, we expand this claim to general $\epsilon_{t, t^{\prime}}$ where $t^{\prime}> t$. We state the continual learning error between subsequent tasks using this analysis technique. 
\begin{restatable}{lemma}{singleperturb}
    \label{lem:singleperturb}
        Let $\bar{\lambda} = \underset{l \in [L], i, j \in [T]}{\max}\lambda_{i, j}^l$. For any input vector from the dataset for the $t$th task $x \in \mathcal{D}_t$, the $\ell_2$ norm of the difference of the outputs from models $\mathbf{M}_t$ and $\mathbf{M}_{t+1}$ are upper bounded by $\forall x \in \mathcal{D}_t$,
    \begin{align}
        \mathbb{E}\bigg[\|\mathbf{M}_t&(x) - \mathbf{M}_{t+1}(x) \|_2 \bigg]\leq \nonumber\\
    &L2^L\chi\bar{\lambda} \left(\prod_{l=1}^L L_l\|\mathbf{A}_{t, l}\|_2 \right) \gamma W^{-\beta} \alpha^{\frac{1-2\beta}{2}}\text{.}\nonumber
    \end{align}
    Here, $\chi$ denotes the maximum norm of the input in $\mathcal{D}_t$, i.e. $\chi = \underset{x \in \mathcal{D}_t}{\max} \|x\|_2$.
\end{restatable}
This bound states that the difference between outputs in two models can be tracked as the input flows through the models. If the two models have close weights at each layer, they will also have close outputs. Here, we see how width and sparsity cause the layers to be similar, resulting in low continual learning error between two sequential models. Using such a bound, we can prove that the difference between the original and the new model differs only so much. Moreover, we can extend the above to compare models more than one task apart, i.e., comparing $\mathbf{M}_t$ and $\mathbf{M}_{t^{\prime}}$. Doing so yields our final theorem. The implicit regularization of training large-width networks to find minima close to initialization acts as a functional regularization in the continual learning setting on the order of $W^{-\beta}$.

\subsection{Extension to Noise Stability}
\label{sec:noisestable}
The dependence on the weight norm is pessimistic in our bounds since neural networks tend to be stable to noise in the input \citep{pmlr-v80-arora18b}. To remove such dependence, we can rely on two terms: layer cushion and activation contraction. 
\begin{restatable}{definition}{layercushion}
    \label{def:layer cushion}
    Let $x_{t, l-1}$ be defined as the output of the first $l-1$ layers of $\mathbf{M}_t$ for some input $x \in \mathcal{D}_t$. The layer cushion of layer $l$ of the model $\mathbf{M}_t$ on the task $t$ is defined to be the smallest number $\mu_{t, l}$ such that for all inputs in $x \in \mathcal{D}_t$, we have that 
    $\|\mathbf{A}_{t, l}\|_2 \|\phi_l(x_{t, l-1}) \|_2 \leq \mu_{t, l}\|\mathbf{A}_{t, l}\phi_l(x_{t, l-1})\|_2 \text{.}$
\end{restatable}
Intuitively, the layer cushion constant $\mu_{t, l}$ is a data-dependent constant that tightens the pessimistic analysis of using the weight norms. Moreover, we will define the activation contraction similarly as the following.
\begin{restatable}{definition}{activationcontraction}
    \label{def:activation contraction}
    Let $x_{t, l-1}$ be defined as the output of the first $l-1$ layers of $\mathbf{M}_t$ for some input $x \in \mathcal{D}_t$. The activation contraction $c_{t}$ for layer $l$ and model $\mathbf{M}_t$ is defined as the smallest number such that for any layer $l$ and any $x \in \mathcal{D}_t$ such that
    $\|x_{t, l-1}\|_2 \leq c_t\|\phi_l(x_{t, l-1})\|_2\text{.}$
\end{restatable}
This connection tightens the pessimistic analysis surrounding the activation layers. Combining these two definitions, we can remove the dependence on the weight norm in our bounds. 
\begin{restatable}{theorem}{smoothenedproof}
\label{thm:smoothenedproof}
    Denote $\Gamma_t = \underset{x \in \mathcal{D}_t}{\max}  \|\mathbf{M}_t(x)\|_2$. Then, we can characterize the continual learning error between two subsequently trained models as
    \begin{align}
    \mathbb{E}\left[\|\mathbf{M}_t(x) - \mathbf{M}_{t^{\prime}}(x) \|_2\right] &\leq \Gamma_t(t^{\prime} - t)\gamma\bar{\lambda} W^{-\beta} \alpha^{\frac{1-2\beta}{2}}\eta \text{,} \nonumber
    \end{align}
    where $\eta = \left(\prod_{i=1}^l \kappa_i +\kappa_i(t^{\prime} - t)\gamma\bar{\lambda}\mu_{t,i} \right)\left(\sum_{i=1}^l \kappa_i\right)$ and $\kappa_i = L_ic_i\mu_{i, t}$.
\end{restatable}

Empirically, such constants can improve upon the pessimistic analysis using weight norms. For more details on the empirical value of these constants, please see \citet{pmlr-v80-arora18b}.
\section{Experiments}

\begin{figure*}[t!]
\centering
\subfloat[.48\textwidth][Rotated MNIST]{

  \begin{sc}
  \begin{tabular}{ccccc}
    \toprule
    Width & AA & AF & LA & JA \\
    \midrule
    $32$ & 56.3 & 37.7 & 93.0 & 91.8\\
    $64$ & 58.7 & 36.0 & 93.5 & 93.5 \\
    $128$ & 59.8 & 35.0 & 93.8 & 94.3 \\
    $256$  & 60.9 & 34.2 & 94.0 & 94.8 \\
    $512$ & 61.9 & 33.2 & 94.1 & 95.0 \\
    $1024$ & 62.7 & 32.6 & 94.2 & 95.3 \\
    $2048$ & 64.1 & 31.2 & 94.3 & 95.5 \\
    $4096$ & 65.3 & 30.2 & 94.5 & 95.7 \\
    $8192$ & 66.7 & 28.9 & 94.7 & 95.7 \\
    $16384$ & 68.0 & 27.9 & 94.9 & 95.9 \\
    $32768$ & 69.4 & 26.6 & 95.6 & 96.1 \\
    $65536$ & 69.6 & 26.7 & 95.6 &  96.2 \\
    \bottomrule
  \end{tabular}
  \end{sc}
  }
\subfloat[.48\textwidth][Rotated Fashion MNIST]{
  \begin{sc}
  \begin{tabular}{ccccc}
    \toprule
    Width & AA & AF & LA & JA \\
    \midrule
    $32$ & 37.7 & 46.0 & 82.1 & 77.8 \\
    $64$ & 37.9 & 46.0 & 82.4 & 80.0 \\
    $128$ & 38.2 & 46.0 & 82.5 & 79.4 \\
    $256$ & 38.4 & 45.9 & 82.7 & 79.8 \\
    $512$ & 38.8 & 45.6 & 82.9 & 79.9 \\
    $1024$ & 39.3 & 45.3 & 83.1 & 79.9 \\
    $2048$ & 39.9 & 44.8 & 83.3 & 79.1 \\
    $4096$ & 40.1 & 44.9 & 83.7 & 80.9 \\
    $8192$ & 40.8 & 44.5 & 83.9 & 80.2 \\
    $16384$ & 41.4 & 44.3 & 84.5 & 78.8 \\
    $32768$ & 41.9 & 44.3 & 84.9 & 79.9 \\
    $65536$ & 42.0 & 44.6 & 85.5 & 80.9 \\
    \bottomrule
  \end{tabular}
  \end{sc}
}

\captionof{table}{Our Continual Learning experiments on varying width FFNs on Rotated MNIST and Rotated Fashion MNIST. We see that the Average Forgetting slowly stops decreasing after a width of $2^{10}$.}
\label{tab:continual learning experiments}
\end{figure*}

To empirically validate the theoretical results in this work, we conducted a series of experiments across multiple established continual learning benchmarks following prior work  \citet{Mirzadeh2022, mirzadeh2022architecture}. We find that while our theoretical framework does not contain all information about catastrophic forgetting, it is roughly accurate in its predictions. We conduct all experiments on a single AWS g5g.8xlarge instance equipped with an NVIDIA A10 GPU.

\subsection{Datasets and Metrics}

Following the literature, we train models on four datasets: Rotated MNIST, Rotated SVHN, Rotated Fashion MNIST, and Rotated GTSRB. To construct each rotated dataset, we construct five tasks comprising of the original images from the dataset rotated by $0$, $22.5$, $45$, $67.5$, and $90$ degrees. We then train on each of these tasks in sequence. 
Following the work of \citet{Mirzadeh2022} we evaluate the efficacy of a continual learning model via four metrics: \textit{Average Accuracy} (AA), \textit{Average Forgetting} (AF), \textit{Learning Accuracy} (LA), and \textit{Joint Accuracy} (JA). For more discussion on the meaning of these metrics, please see \Cref{sec:metrics}. 

\begin{figure*}[t!]
    \subfloat[1 Layer]{\includegraphics[width=0.3\textwidth]{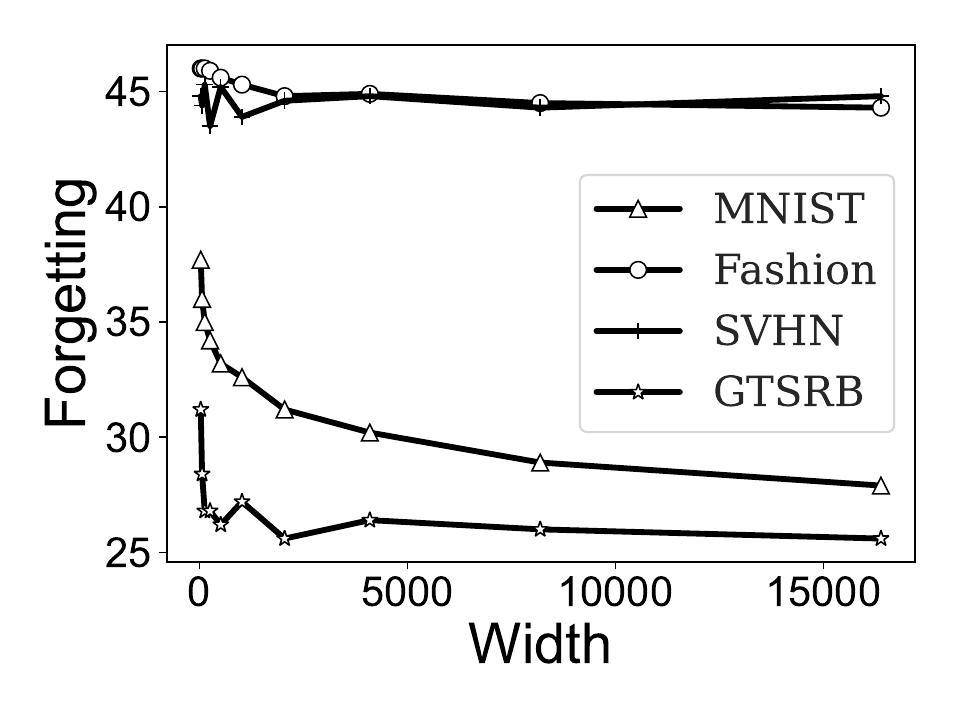}\label{fig:layer1width}}\hfill
    \subfloat[2 Layer]{\includegraphics[width=0.3\textwidth]{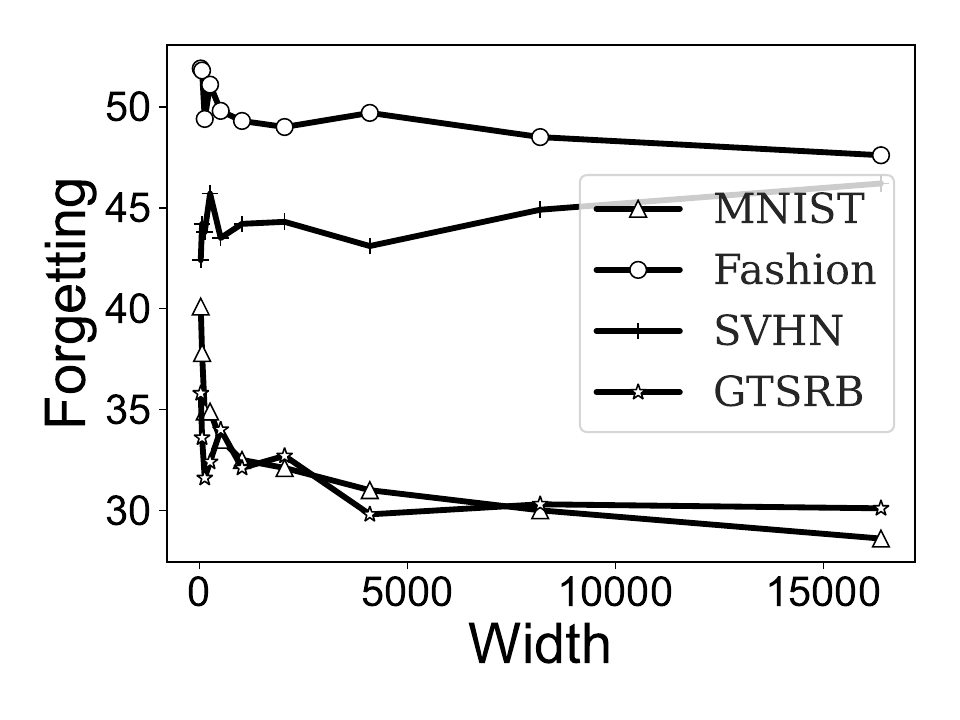}\label{fig:layer2width}}\hfill
    \subfloat[3 Layer]{\includegraphics[width=0.3\textwidth]{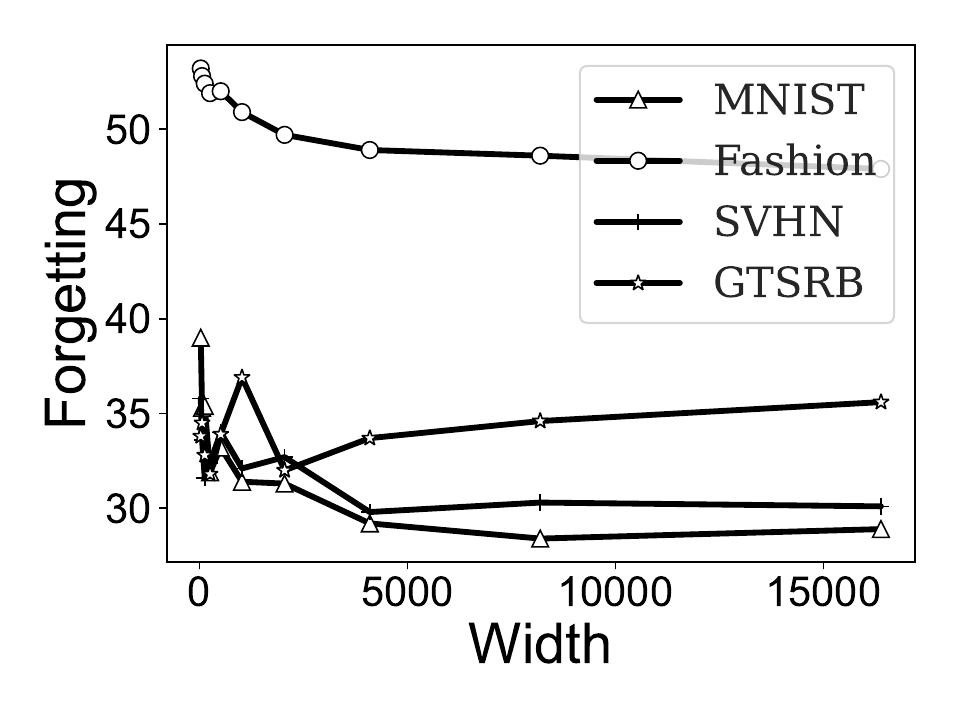}\label{fig:layer3width}}
   
    \caption{We visualize the diminishing returns of increasing width across networks of varying depth. This corroborates our theoretical analysis.  }
     \label{fig:erroroverwidth}
\end{figure*}

\subsection{Modeling Setup}

We utilize the same model architecture for all of the experiments: a Feed Forward Network consisting of an input, hidden, and output layer. Between each layer, we use ReLU activations. We vary the width of the hidden layer in all experiments while keeping all other hyperparameters fixed. We train these models using the SGD optimizer 
 as well as Adam \citep{kingma2015adam}, the latter of which we report in the appendix. We train these models as classifiers on each task using CrossEntropy Loss for $5$ epochs. This training paradigm is a standard modeling choice in the Continual Learning literature \citep{mirzadeh2020linear}.
 \footnote{These MLP models perform well short of state-of-the-art convolutional neural networks \citep{lecun1995convolutional} and vision transformer \citep{dosovitskiy2020transformers} architectures. Nevertheless, our work focuses on developing a rigorous and principled understanding of the effect of width on the training dynamics of continual learning.}. 
 To note, while swapping out the input and output layers for each task is standard, we found that doing so did not impact the results significantly since each task has the same input and output dimensionality, so we do not swap the input and output layers for each task to better align with the theoretical analysis. 
\section{Results}
\begin{figure*}[hbtp!]
    \subfloat[Forgetting vs. Depth]{
        \includegraphics[width=0.48\textwidth]{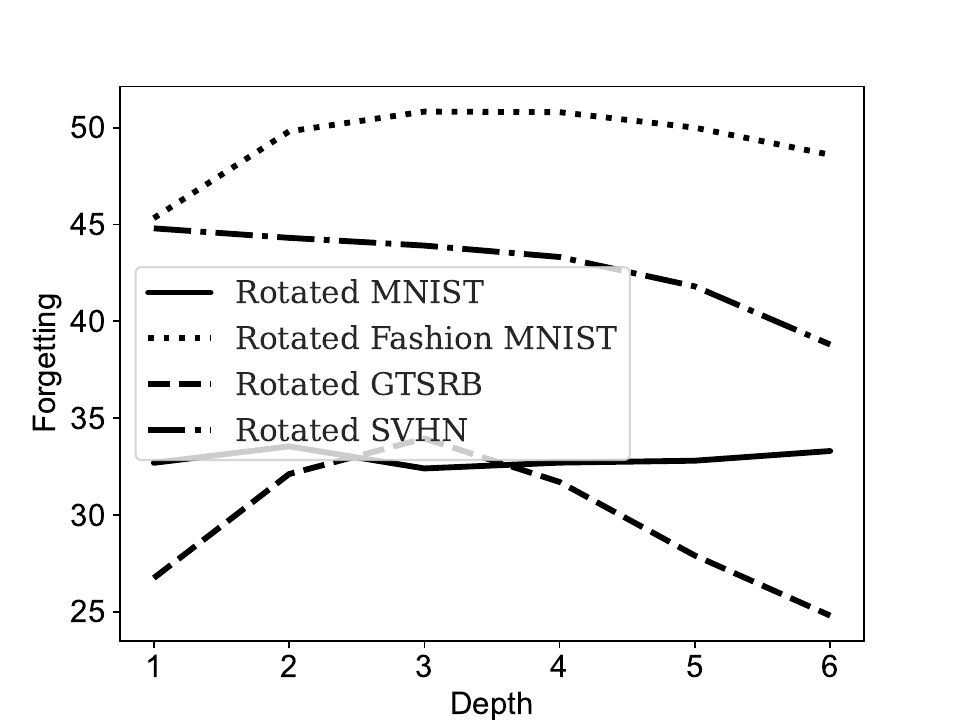}
        \label{fig:mnisterroroverdepth}
        }
    \subfloat[Forgetting vs. Task Index]{
        \includegraphics[width=0.48\textwidth]{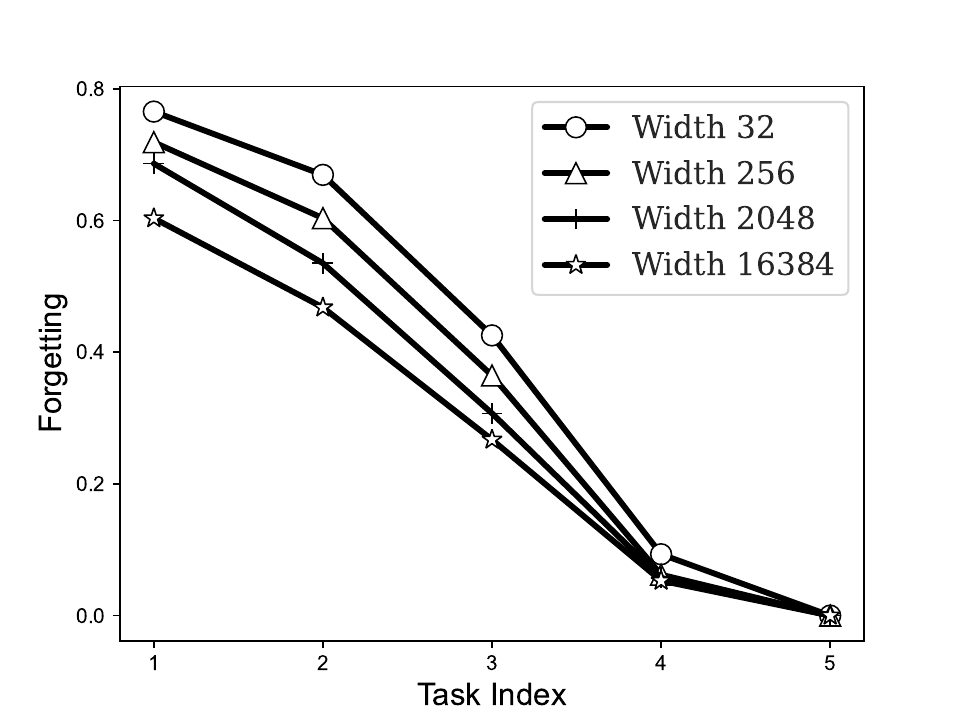}
        \label{fig:mnisterrorovertime}
    }
    \caption{We visualize the relationship between depth and task index on forgetting over several datasets. \Cref{fig:mnisterrorovertime} is only for Rotated MNIST.}
    \label{fig:erroroverdepth}
\end{figure*}
\paragraph{Distance From Initialization}\label{sec:distfrominit} In \Cref{fig:distmnist} and \Cref{fig:distcifar}, we note that a model's distance from initialization tends to decrease as a function of width, which provides empirical evidence of \Cref{ass:distance}. The numerator of our relative distance metric is the Frobenius norm of the difference between the hidden layers of the intermediate models trained on the first task and second tasks. In the denominator of our metric, we have the $\ell_2$ norm of the hidden layer of the model trained on the first task. We see a slowly decreasing distance as the width increases exponentially. To better understand the constants in \Cref{ass:distance}, we find what values of $\gamma$ and $\beta$ best explain the curves seen in \Cref{fig:distmnist} and \Cref{fig:distmnist}. We see values of $\gamma = 2.5, \beta = 0.12$ for Fashion MNIST and $\gamma=0.013, \beta=0.311$ for MNIST. This finding also validates the existing literature on the relationship between width and distance from initialization \citep{nagarajan2019generalization, mirzadeh2020linear}. We plot the predicted line denoting the predicted relationship from \Cref{ass:distance} using the best fitting parameters. 

\paragraph{Dense Model Results} Furthermore, in \Cref{tab:continual learning experiments} and \Cref{fig:erroroverwidth}, we see that increasing width helps with improving forgetting early on but receives diminishing returns eventually. Looking specifically at \Cref{fig:erroroverwidth}, we see a similar trend for datasets MNIST, Fashion MNIST, and GTSRB. Namely, we see that the increasing width decreases forgetting as width increases, but the forgetting plateaus at larger widths. This holds across one-layer, two-layer, and three-layer Feed Forward Networks. However, for SVHN, we see for the first layer that width does not affect the forgetting much and even increases the forgetting as width increases in two and three layer networks. All four of these datasets demonstrate a finding beyond the existing literature: increasing width does not always decrease forgetting. In \Cref{tab:continual learning experiments}, we provide specific numbers for one-layer networks on MNIST and Fashion MNIST. For other datasets and depths, we relegate these tables to the appendix for space purposes; see \Cref{fig:remainingmnist} and \Cref{fig:svhngstrb}. We see specifically in both cases as width approaches $2^{16}$, the Average Forgetting begins to plateau, as predicted by our theory. Despite the initial forgetting for width $2^5$ on both datasets beginning at different levels of forgetting, they both exhibit a similar shape curve in the forgetting, changing much less after a width of $2^{11}$.

\begin{figure*}[ht!]
\centering
\subfloat[.48\textwidth][Rotated SVHN]{
  \begin{sc}
  \begin{tabular}{ccccc}
    \toprule
    Width & AA & AF & LA & JA \\
    \midrule
    1x & 48.8 & 43.8 & 92.6 & 92.7 \\
    5x & 49.0 & 43.3 & 92.3 & 93.6 \\
    10x & 47.9 & 42.1 & 90.0 & 94.2 \\
    15x & 49.6 & 43.1 & 92.7 & 94.0 \\
    20x & 50.6 & 40.1 & 90.7 & 94.6 \\
    \bottomrule
  \end{tabular}
  \end{sc}

}
\subfloat[.48\textwidth][Rotated GTSRB]{
\begin{sc}
  \begin{tabular}{ccccc}
    \toprule
    Width & AA & AF & LA & JA \\
    \midrule
    1x & 43.8 & 45.6 & 89.4 & 89.0 \\
    5x & 50.7 & 42.7 & 93.4 & 91.1 \\
    10x & 52.1 & 41.5 & 93.6 & 92.4 \\
    15x & 55.0 & 39.1 & 94.1 & 93.1 \\
    20x & 53.1 & 40.9 & 94.0 & 89.1 \\
    \bottomrule
  \end{tabular}
  \end{sc}

}

\captionof{table}[foo]{We report the numbers from our continual learning experiments using a Wide ResNet model \cite{zagoruyko2016wide} on the SVHN and GTSRB datasets. The width reported in the first column corresponds to the multiplicative amount applied to the width factor parameter of the Wide ResNet model. We again see a similar trend of diminishing returns as we did in the MLP setting.}
\label{fig:resnet}
\end{figure*}
 
\paragraph{Forgetting over Time} We briefly explore the error over tasks: how much the final model forgets as the number of tasks seen increases. We take the model trained on the final task and compute its forgetting on each of the previous tasks. For example, the forgetting on task index $5$ is $0$, and the forgetting on task index $1$ is the largest since that is the oldest dataset the model has seen. For space, we include the majority of the figures in the appendix. We plot the error over the task indices in \Cref{fig:mnisterrorovertime} and \Cref{fig:errorovertime} for FFNs of differing widths on all datasets. We see a roughly linear increase in forgetting as the task index decreases, as is expected. This relationship is seemingly independent of width and dataset. This linear relationship independent of width corroborates the theoretical analysis in \Cref{lem:ourperturbation}, which predicts a linear relationship. We include further details in the appendix.

\paragraph{Connection between Forgetting and Depth} Our model roughly predicts an increase in forgetting as the depth of the neural network is increased. This matches several empirical results such as the results from \citet{mirzadeh2022architecture} that see an increase in forgetting as depth increases. We reproduce these experiments and report them in \Cref{fig:mnisterroroverdepth}. Up to a depth of $3$, We see empirically that depth and forgetting are positively correlated, roughly corroborating our theory. Specifically, on Fashion MNIST and GTSRB, increasing the number of layers from $1$ to $3$ roughly increases the forgetting. On datasets MNIST and SVHN, the impact of depth on forgetting is too small to make any meaningful connection. However, for all datasets besides MNIST, we see forgetting decreases as depth is increased after $4$ hidden layers. We attribute this to the lower overall accuracy at these higher depths due to vanishing gradients. At these high depths, vanishing gradients cause the overall accuracy to decrease. As accuracy decreases in the high-depth regime, the forgetting artificially decreases. 

\paragraph{Connection between Forgetting and Sparsity} We briefly note that our theory predicts that inducing sparsity will improve forgetting. In \Cref{tab:sparsity}, when setting the sparsity coefficient $\alpha=0.1$, forgetting decreases significantly. For example, on MNIST, at width $2^{14}$, the Average Forgetting decreases to $0.2$, compared to $27.9$ in the dense setting. This corroborates our analysis's prediction of the connection between sparsity and forgetting.

\subsection{Experiments with Wide ResNet Models}

In this section, we present the results from replacing the FFN used in our primary experimental analysis with a Wide Resnet model \cite{zagoruyko2016wide}. While our theoretical results focus solely on MLP models, we also investigate whether our findings of dimishing returns of width in continual learning also hold in the case of more sophisticated convolutional networks, which achieve higher accuracy and are much more commonly used in practice for vision tasks \cite{krizhevsky2012imagenet}. Following prior work from \cite{Mirzadeh2022}, we focus on the Wide Resnet architecture due to the model's strong empirical performance and the ability to easily scale the model's width through a predefined multiplicative factor. In our experiments, we focus on the same rotated SVHN and GTSRB tasks considered earlier and maintain all other settings other than the model architecture. Our findings are documented in \Cref{fig:resnet}. For both datasets, we see that as the width is increased greatly, the forgetting slowly decreases to $40.1$ and $40.9$ for SVHN and GTSRB respetively. From these results, we do indeed see a similar diminishing returns phenomenon in forgetting as we scale the width factor of the ResNet model. These preliminary findings also suggest that extending our theoretical results to other model architectures more commonly found in practice might be a fruitful direction for future work. 

\section{Discussion}
We examine both empirically and theoretically the connection between continual learning and the width of a Feed Forward Network. While our theoretical framework doesn't completely capture all information about catastrophic forgetting, it is a valuable first step to analyzing continual learning theoretically. Increasing width receives diminishing returns at some point. We empirically see these diminishing returns at larger hidden dimensions than tested in previous literature. A possible extension would be whether similar increases in scale receive diminishing returns such as depth, number of channels in CNNs, hidden dimension in LLMs, etc. Additionally, an important work could be to examine if width, in conjunction with other functional regularization methods, can reduce the effect of diminishing returns.

\paragraph{Limitations} Our analysis relies on an assumption about the distance from initialization seen during training, which we have yet to prove rigorously. Moreover, our analysis is restricted to FFNs and has yet to be extended to more complex architectures such as Residual Networks or Attention networks. Furthermore, our analysis does not perfectly predict the exact continual learning error.  

\section*{Impact Statement}
This paper presents work whose goal is to advance the field of Machine Learning. There are many potential societal consequences of our work, none of which we feel must be specifically highlighted here.

\bibliography{iclr2024_conference}
\bibliographystyle{icml2024}

\newpage
\appendix
\onecolumn


\begin{table}
\centering
\begin{sc}

  \begin{tabular}{ccc}
    \toprule
    Notation & Meaning & Location \\
    \midrule
    $\mathbf{M}_t$ & Model for $t$th task & \Cref{sec:notation}\\
    $L_i$ & Lipschitz-Constants if $i$th activation function & \Cref{sec:notation} \\
    $W$ & Width of model $\mathbf{M}$ & \Cref{sec:notation}\\
    $\mathbf{A}_{t, l}$ & $l$th layer of the $t$th model & \Cref{sec:notation} \\
    $\mathcal{D}_t$ & $t$th training dataset & \Cref{sec:problem_setup} \\
    $\epsilon_{t, t^{\prime}}$ & Continual learning error of $t^{\prime}$th model on $\mathcal{D}_t$ & \Cref{sec:problem_setup} \\
    $\lambda_{i,j}^l$ & $\frac{\|\mathbf{A}_{l, j}[\mathcal{A}_{l, i}]\|_2}{\|\mathbf{A}_{l, i}[\mathcal{A}_{l, i}]\|_2}$ & \Cref{lem:ourperturbation}\\
    $\bar{\lambda}$ & Maximum ratio of spectra norms of all models and layers & \Cref{lem:ourperturbation} \\
    $L$ & Depth of model &\Cref{sec:notation} \\
    $\gamma, \beta$ & Constants governing the distance from initialization & \Cref{ass:distance} \\
    $\chi$ & $\underset{x \in \mathcal{D}_t}{\max} \|x\|_2$ & \Cref{lem:singleperturb} \\
    $\mu_{t, l}$ & Layer Cushion for layer $l$ of model $t$ & \Cref{def:layer cushion} \\
    $c_t$ & Activation Contraction of layer $l$ & \Cref{def:activation contraction} \\
    $\Gamma_t$ & $\underset{x \in \mathcal{D}_t}{\max}  \|\mathbf{M}_t(x)\|_2$ & \Cref{thm:smoothenedproof}\\
    \bottomrule
  \end{tabular}
  \captionof{table}{Table of Notation used in Paper}
  \end{sc}
\end{table}

\section{Metrics}
\label{sec:metrics}
Average accuracy is defined as the mean test accuracy of the final model over all the tasks after training on all tasks in sequence. Average forgetting is calculated as the mean difference over tasks between the accuracies obtained by the intermediate model trained on that task and the final model. The learning accuracy of a given task measures the accuracy that the model achieves immediately after training on that task. We also report learning accuracy as an average over all tasks. Finally, we report Joint Accuracy, which is the accuracy of training a model on all of the combined datasets. 

\section{Additional Continual Learning Experiments}
\begin{figure*}[t!]
    \centering
    
    \subfloat[GTSRB]{\includegraphics[width=0.48\textwidth]{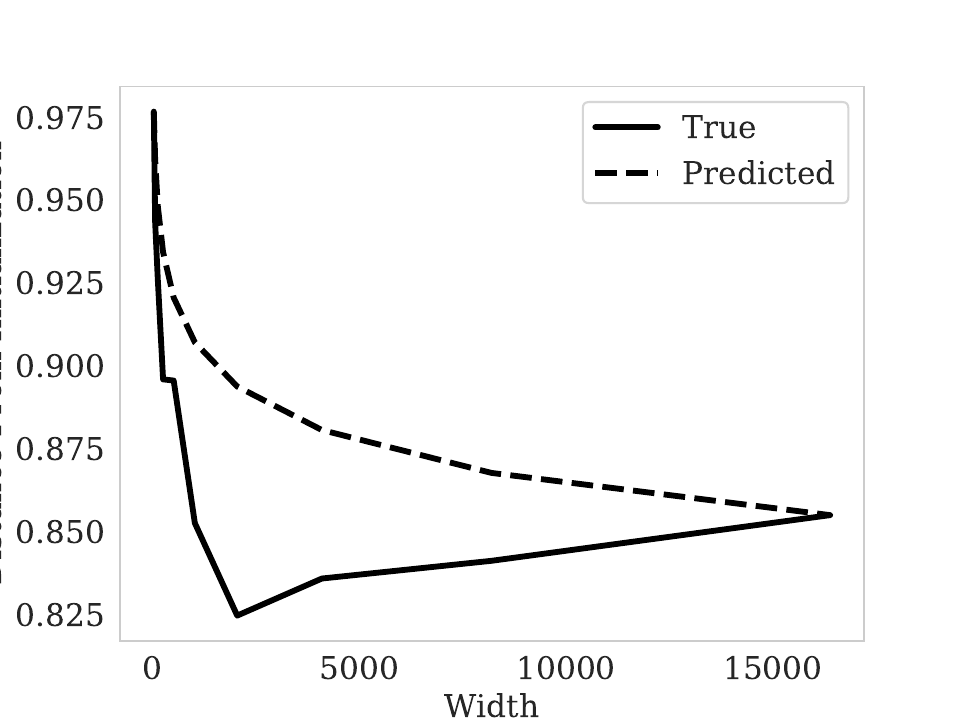}}\hfill
    \subfloat[SVHN]{\includegraphics[width=0.48\textwidth]{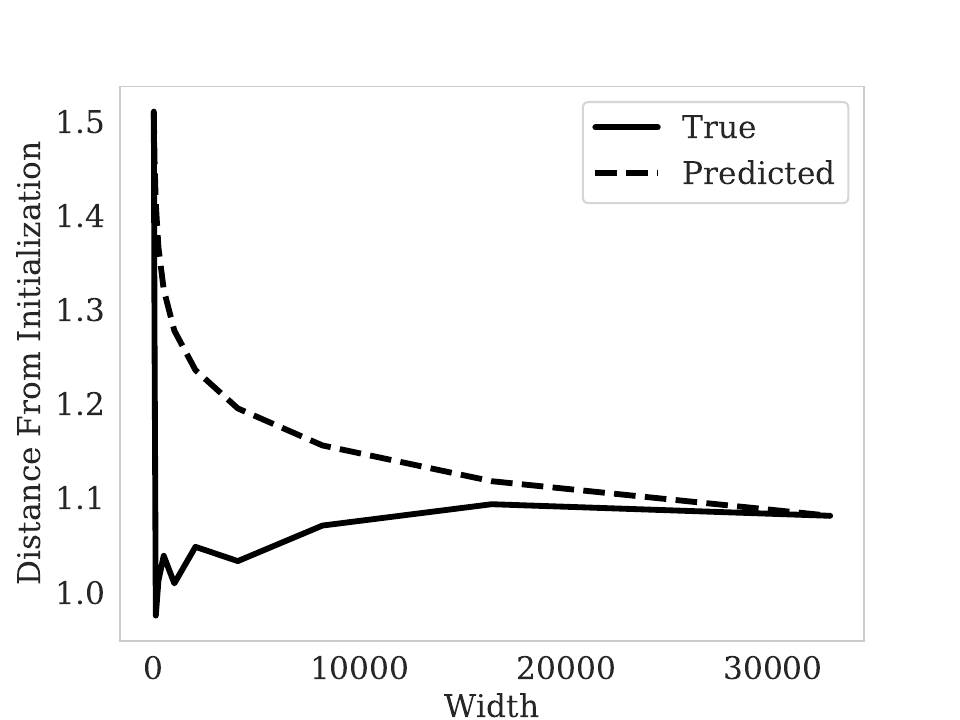}}
    
    \caption{We plot the distance from initialization for both GTSRB and SVHN experiments. We see that distance from initialization decreases slowly as the width is increased for both datasets.}
\end{figure*}

Here, we report several additional experiments to investigate the relationship between width and the ability to learn continually. We demonstrate the relationship between SGD and this diminishing returns phenomenon by repeating the experiments with Adam in \Cref{sec:addsgd}. Moreover, we test how inducing row-wise sparsity affects the relationship between Continual Learning ability. Increasing the sparsity can significantly decrease the average forgetting, as predicted by our theoretical analysis. We demonstrate this in \Cref{sec:addsparsity}.

Furthermore, we study the relationship between forgetting and the number of tasks trained. We report on this in \Cref{sec:addnumtasks}. 

\subsection{Adam and Diminishing Returns}
\label{sec:addsgd}
To further investigate the connection between width and continual learning, we wish to see whether the phenomenon of diminishing returns is visible in other optimization algorithms. In our theoretical analysis, we argue that the implicit regularization of gradient-based optimizers finds minima closer to previous tasks, improving continual learning ability. In our experiments in the main section, we conduct all experiments with SGD and empirically observe these diminishing returns. We wish to examine if this implicit regularization and diminishing returns are unique to SGD or generalizes to other optimization algorithms. To do so, we repeat the continual learning experiments in the main text with Adam as the optimizer. We report our results in \Cref{tab:continual learning experiments sgd}. Indeed, we see very similar results to SGD as in \Cref{tab:continual learning experiments sgd}. We again observe the diminishing returns for both datasets with Adam as the optimizer. This observation suggests that our analysis and this phenomenon are not unique to SGD and extend to different optimizers. 
\begin{figure*}[h!]
\centering
\subfloat[.48\textwidth][Rotated MNIST]{
  \begin{sc}
  \begin{tabular}{ccccc}
    \toprule
    Width & AA & AF & LA & JA \\
    \midrule
    32 & 51.3 & 43.9 & 94.6 & 91.1 \\
    64 & 51.5 & 43.9 & 94.7 & 92.5  \\
    128 & 52.2 & 43.2 & 94.8 & 94.1\\
    256 & 51.4 & 44.0 & 94.6 & 94.0\\
    512 & 56.3 & 39.3 & 94.9 & 93.7 \\
    1024 & 52.9 & 42.5 & 94.9 & 93.8\\
    2048 & 52.4 & 43.4 & 95.1 & 93.7\\
    4096 & 54.1 & 41.5 & 94.9 & 94.3\\
    8192 & 53.6 & 42.0 & 95.1 & 94.1 \\
    16384 & 51.0 & 44.7 & 95.0 & 93.7\\
    32768 & 55.2 & 40.7 & 95.2 & 94.0\\
    65536 & 54.9 & 41.0 & 95.4 & 93.1\\
    \bottomrule
  \end{tabular}
  \end{sc}
  \centering
}

\centering
\subfloat[.48\textwidth][Rotated Fashion MNIST]{
  \begin{sc}
  \begin{tabular}{ccccc}
    \toprule
    Width & AA & AF & LA & JA \\
    \midrule
    32 & 33.95 & 52.01 & 84.57 & 77.8 \\
    64 & 33.66 & 53.04 & 85.48 & 79.7 \\
    128 & 35.64 & 51.73 & 86.18 & 79.4 \\ 
    256 & 33.3 & 54.19 & 86.32 & 79.8 \\
    512 & 34.27 & 53.39 & 86.52 & 79.9 \\
    1024 & 36.31 & 51.33 & 86.62 & 79.1 \\
    2048 & 35.73 & 51.77 & 86.53 & 80.9 \\
    4096 & 37.47 & 50.2 & 86.66 & 80.2 \\
    8192 & 33.25 & 54.32 & 86.37 & 78.8 \\
    16384 & 33.98 & 53.71 & 86.57 & 79.9 \\
    32768 & 33.95 & 53.65 & 86.49 & 80.9 \\
    65536 & 36.91 & 50.69 & 86.53 & 79.0\\
    \bottomrule
  \end{tabular}
  \end{sc}
  \centering
}

\captionof{table}{Continual Learning Results with Adam optimizer training. We see that despite changing the optimizer to Adam, we see very similar results, including the diminishing returns trend. This suggests that the diminishing returns phenomenon is not optimizer-dependent.}
\label{tab:continual learning experiments sgd}
\end{figure*}

\subsection{Sparsity and Width}
\label{sec:addsparsity}

In Table \ref{tab:sparsity}, we report the results of our investigation into the power of row-wise sparsity to mitigate the effects of catastrophic forgetting, as predicted by our theory. For these experiments, we utilize the BOLT deep learning library (\citet{meisburger2023bolt}) which allows for explicitly setting row-wise activatins. To summarize our earlier discussion, we define row-wise sparsity as activated only a subset of the hidden layer neurons in our feedforward model architecture. We select a fraction $\alpha$ of the hidden layer neurons for each task to be active uniformly at random. Note that some neurons may be shared across multiple tasks, which leads to some amount of forgetting, but at a significantly reduced rate to what we observed in the dense case. Moreover, the overall learning accuracy is not decreased despite fewer neurons being used during training. In these experiments, we also apply sparsity to the output layer such that each task has a distinct set of output classes. As predicted by our theory, we observe that sparsity can counterweight the diminishing returns on reducing forgetting purely through width alone. However, some effects of diminishing returns can still be seen in both datasets. 

\begin{figure*}[h!]
\centering
\subfloat[.48\textwidth][Rotated MNIST]{

  \begin{sc}
  \begin{tabular}{ccccc}
    \toprule
    Width & AA & AF & LA & JA \\
    \midrule
    32 & 68.5 & 5.4 & 73.9 & 93.8\\
    64 & 75.6 & 13.3 & 88.9 & 95.8\\
    128 & 82.0 & 9.8 & 91.8 & 96.7 \\
    256  & 86.2 & 6.8 & 83.1 & 97.3 \\
    512 & 91.2 & 3.3 & 94.5 & 97.7 \\
    1024 & 94.6 & 1.3 & 95.9 & 97.9 \\
    2048 & 95.5 & 1.3 & 96.7 & 98.1 \\
    4096 & 96.4 & 1.0 & 97.3 & 98.0 \\
    8192 & 97.1 & 0.5 & 97.6 & 98.1 \\
    16384 & 97.7 & 0.2 & 97.9 & 98.0 \\
    32768 & 97.6 & 0.3 & 98.0 & 98.2 \\

    \bottomrule
  \end{tabular}
  \end{sc}
  }
\subfloat[.48\textwidth][Rotated Fashion MNIST]{
  \begin{sc}
  \begin{tabular}{ccccc}
    \toprule
    Width & AA & AF & LA & JA \\
    \midrule
    32 & 55.7 & 6.0 & 61.7 & 77.8 \\
    64 & 58.7 & 13.9 & 72.6 & 79.7 \\
    128 & 74.6 & 6.0 & 80.6 & 79.4 \\
    256 & 69.9 & 12.7 & 82.6 & 79.8 \\
    512 & 76.5 & 7.5 & 84.1 & 79.9 \\
    1024 & 79.1 & 5.7 & 84.8 & 79.1 \\
    2048 & 77.7 & 7.6 & 85.4 & 80.9 \\
    4096 & 82.3 & 3.7 & 85.9 & 80.2 \\
    8192 & 79.6 & 6.9 & 86.5 & 78.8 \\
    16384 & 78.4 & 8.5 & 86.8 & 79.9 \\
    32678 & 78.4 & 8.5 & 86.8 & 80.9\\ 
    \bottomrule
  \end{tabular}
  \end{sc}
  \centering
  }
\captionof{table}{Continual Learning experiments with Row-Wise Sparsity with $\alpha=0.1$. We see that increasing row-wise sparsity can significantly decrease forgetting while not decreasing overall learning accuracy. This corroborates our theoretical results. We still do see diminishing returns in terms of increasing width and continual learning.}
\label{tab:sparsity}
\end{figure*}

\subsection{Accuracy over number of tasks}
\begin{figure*}[t!]
    \subfloat[Fashion MNIST]{\includegraphics[width=0.48\textwidth]{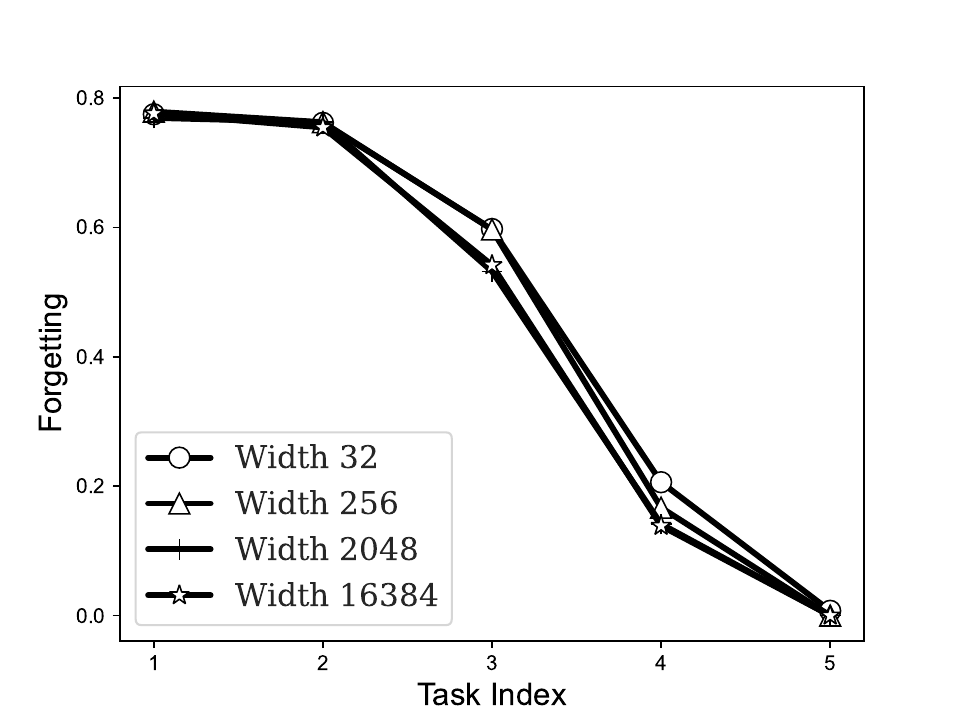}\label{fig:16384}}\hfill
    \subfloat[GTSRB]{\includegraphics[width=0.48\textwidth]{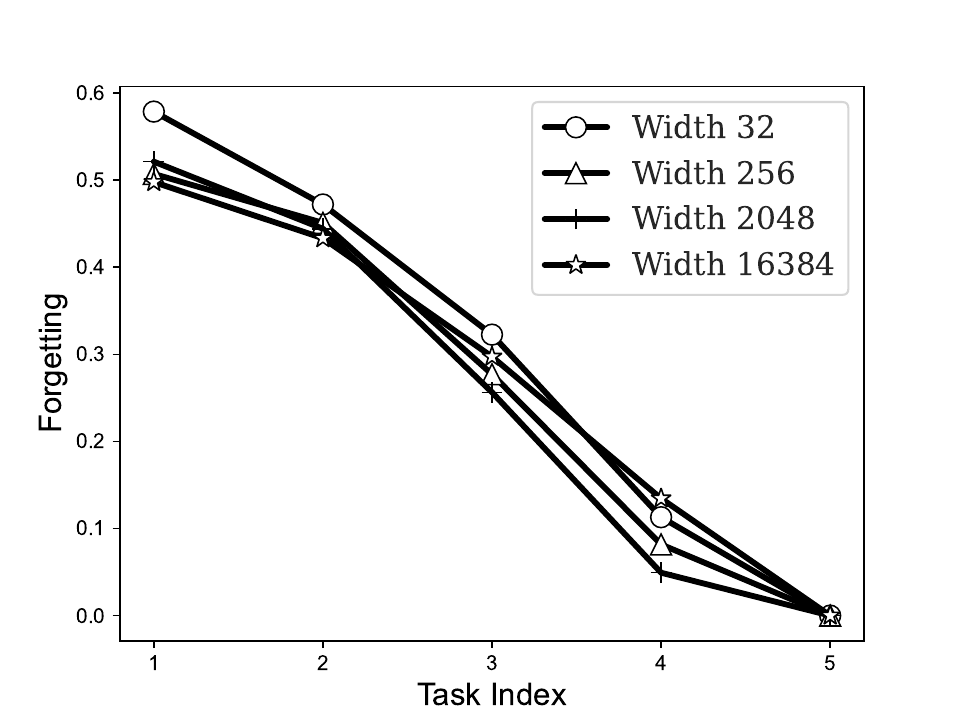}\label{fig:65535}}
   
     \label{fig:errorovertime}
\end{figure*}
\label{sec:addnumtasks}
We measure the accuracies of the final learned model over all tasks and plot how the accuracy decreases. We do this over GTSRB and Fashion MNIST. We conducted this experiment over several different width networks. Our analysis predicts that there should be a roughly linear relationship between the accuracy and the number of tasks. With more intermediate tasks between an initial task and the final task, the model will decrease roughly linearly in accuracy.

Moreover, our analysis predicts that this error should be independent of the width. We report the results of our experiments in \Cref{fig:errorovertime}. We see in both experiments that the accuracy roughly decreases as the number of tasks increases. Moreover, this relationship holds across all the widths tested. This observation corroborates our theory. 

\subsection{Continual Learning Experiments on Additional Datasets and Depth}
We report the numbers from all remaining experiments. In \Cref{fig:remainingmnist}, we report all additional experiments on Fashion MNIST and MNIST. We also report numbers for two and three layer networks. We also report the numbers for one, two, and three layer networks on datasets SVHN and GTRSB in \Cref{fig:svhngstrb}. We see that the trend of diminishing returns in terms of width across many of the datasets. We also see that this trend remains apparent among different layer networks. Moreover, we see that as depth increasing, the forgetting numbers tend to increase across many of the datasets. 

\begin{figure*}[ht!]
\centering
\subfloat[.48\textwidth][Rotated MNIST (2 Layers)]{
  \begin{sc}
  \begin{tabular}{ccccc}
    \toprule
    Width & AA & AF & LA & JA \\
    \midrule
    32 & 54.8 & 40.1 & 93.7 & 93.1 \\
    64 & 57.5 & 37.8 & 94.3 & 94.7\\ 
    128 & 61.0 & 34.9 & 94.9 & 95.6 \\
    256 & 61.3 & 34.9 & 95.1 & 96.1 \\
    512 & 62.9 & 33.5 & 95.4 & 96.5\\
    1024 & 64.0 & 32.5 & 95.5 & 96.7 \\
    2048 & 64.5 & 32.1 & 95.6 & 96.6 \\
    4096 & 65.7 & 31.0 & 95.8 & 96.8 \\
    8192 & 66.9 & 30.0 & 95.8 & 97.0 \\
    16384 & 68.4 & 28.6 & 96.2 & 97.1 \\ 
    32768 & 67.2 & 30.2 & 96.5 & 97.1 \\
    \bottomrule
  \end{tabular}
  \end{sc}
}
\subfloat[.48\textwidth][Rotated Fashion MNIST (2 Layers)]{
\begin{sc}
  \begin{tabular}{ccccc}
    \toprule
    Width & AA & AF & LA & JA \\
    \midrule
    32 & 32.1 & 51.9 & 82.4 & 81.1 \\
    64 & 32.8 & 51.8 & 82.9 & 82.2 \\
    128 & 35.2 & 49.4 & 83.1 & 83.4 \\
    256 & 33.9 & 51.1 & 83.5 & 84.5 \\ 
    512 & 35.6 & 49.8 & 83.8 & 84.5 \\
    1024 & 36.2 & 49.3 & 83.6 & 84.6 \\
    2048 & 36.5 & 49.0 & 84.3 & 85.0 \\
    4096 & 36.2 & 49.7 & 84.6 & 85.1 \\
    8192 & 38.0 & 48.5 & 85.1 & 86.0 \\
    16384 & 38.9 & 47.6 & 85.3 & 86.1 \\
    32768 & 39.2 & 47.8 & 85.8 & 85.9 \\ 
    \bottomrule
  \end{tabular}
  \end{sc}
}

\subfloat[.48\textwidth][Rotated MNIST (3 Layers)]{
  \begin{sc}
  \begin{tabular}{ccccc}
    \toprule
    Width & AA & AF & LA & JA \\
    \midrule
    32 & 56.5 & 39.0 & 91.8 & 93.1 \\
    64 & 60.6 & 35.3 & 92.6 & 95.2 \\
    128 & 61.2 & 35.4 & 93.8 & 95.9 \\
    256 & 64.9 & 31.9 & 94.4 & 96.3 \\
    512 & 63.7 & 33.2 & 95.5 & 96.5 \\
    1024 & 65.8 & 31.4 & 95.9 & 97.1 \\ 
    2048 & 66.0 & 31.3 & 96.2 & 97.2 \\
    4096 & 68.2 & 29.2 & 96.5 & 97.4 \\
    8192 & 69.2 & 28.4 & 96.7 & 97.3 \\
    16384 & 68.8 & 28.9 & 96.9 & 97.4 \\
    \bottomrule
  \end{tabular}
  \end{sc}
}
\subfloat[.48\textwidth][Rotated Fashion MNIST (3 Layers)]{
\begin{sc}
  \begin{tabular}{ccccc}
    \toprule
    Width & AA & AF & LA & JA \\
    \midrule
    32 & 31.0 & 53.2 & 82.1 & 81.3 \\
    64 & 31.7 & 52.8 & 82.5 & 82.9 \\
    128 & 32.7 & 52.4 & 83.1 & 83.8 \\
    256 & 33.6 & 51.9 & 83.5 & 84.5 \\
    512 & 33.8 & 52.0 & 83.8 & 84.3 \\
    1024 & 35.2 & 50.9 & 84.5 & 85.3 \\
    2048 & 36.6 & 49.7 & 84.6 & 85.5\\
    4096 & 37.6 & 48.9 & 85.1 & 85.9 \\
    8192 & 38.3 & 48.6 & 85.6 & 85.8 \\
    16384 & 39.4 & 47.9 & 85.9 & 85.9\\
    \bottomrule
  \end{tabular}
  \end{sc}
}
\caption{We report the numbers from our experiments on more layers on Fashion MNIST and MNIST. We see that the trend of diminishing returns holds. We also see that increasing depth often leads to an increase in forgetting.}
\label{fig:remainingmnist}
\end{figure*}

\begin{figure*}[ht!]
\centering
\subfloat[.48\textwidth][Rotated SVHN]{
  \begin{sc}
  \begin{tabular}{ccccc}
    \toprule
    Width & AA & AF & LA & JA \\
    \midrule
    32 & 28.8 & 44.8 & 69.8 & 62.2 \\
    64 & 31.5 & 44.4 & 72.6 & 66.5 \\
    128 & 32.7 & 45.3 & 73.7 & 71.1 \\
    256 & 34.8 & 43.5 & 74.9 & 73.0 \\
    512 & 33.6 & 45.2 & 73.9 & 72.0 \\
    1024 & 34.7 & 43.9 & 74.2 & 73.0 \\
    2048 & 34.4 & 44.6 & 74.7 & 72.7 \\
    4096 & 34.7 & 44.8 & 75.0 & 73.5 \\
    8192 & 34.7 & 44.3 & 74.2 & 73.2 \\
    16384 & 33.5 & 44.8 & 73.5 & 70.1 \\
    32768 & 30.3 & 48.4 & 73.8 & 72.9 \\
    65536 & 35.6 & 43.7 & 73.4 & 73.6 \\
    \bottomrule
  \end{tabular}
  \end{sc}
}
\subfloat[.48\textwidth][Rotated GTSRB]{
\begin{sc}
  \begin{tabular}{ccccc}
    \toprule
    Width & AA & AF & LA & JA \\
    \midrule
    32 & 46.3 & 31.2 & 72.4 & 67.1 \\
    64 & 49.9 & 28.4 & 72.8 & 72.1 \\ 
    128 & 50.6 & 26.8 & 72.8 & 68.4 \\ 
    256 & 51.1 & 26.8 & 73.3 & 67.1 \\ 
    512 & 51.4 & 26.2 & 73.4 & 71.5 \\ 
    1024 & 49.9 & 27.2 & 73.0 & 69.2 \\ 
    2048 & 51.7 & 25.6 & 72.8 & 68.4 \\ 
    4096 & 50.4 & 26.4 & 72.6 & 71.0 \\ 
    8192 & 51.98 & 26.0 & 73.8 & 71.3 \\
    16384 & 52.2 & 25.6 & 74.4 & 70.8 \\
    32768 & 54.5 & 24.0 & 74.9 & 72.5 \\
    65536 & 54.5 & 23.8 & 74.3 & 74.1 \\
    \bottomrule
  \end{tabular}
  \end{sc}
}

\subfloat[.48\textwidth][Rotated SVHN (2 Layers)]{
  \begin{sc}
  \begin{tabular}{ccccc}
    \toprule
    Width & AA & AF & LA & JA \\
    \midrule
    32 & 30.6 & 42.4 & 69.9 & 61.0 \\
        64 & 32.5 & 44.2 & 72.1 & 70.4 \\
        128 & 34.9 & 43.8 & 74.4 & 72.9 \\
        256 & 33.4 & 45.7 & 75.7 & 75.7 \\
        512 & 36.5 & 43.5 & 76.7 & 77.2 \\
        1024 & 37.2 & 44.2 & 76.7 & 78.7 \\
        2048 & 36.5 & 44.3 & 77.9 & 77.8 \\
        4096 & 38.3 & 43.1 & 77.5 & 79.1 \\ 
        8192 & 36.8 & 44.9 & 78.9 & 78.4 \\ 
        16384 & 36.0 & 46.2 & 80.0 & 79.6 \\
        32768 & 37.7 & 45.1 & 80.5 & 80.4 \\
    \bottomrule
  \end{tabular}
  \end{sc}
}
\subfloat[.48\textwidth][Rotated GTSRB (2 Layers)]{
\begin{sc}
  \begin{tabular}{ccccc}
    \toprule
    Width & AA & AF & LA & JA \\
    \midrule
    32 & 41.1 & 35.8 & 70.6 & 66.9 \\
    64 & 43.0 & 33.6 & 70.2 & 65.6 \\
    128 & 46.4 & 31.6 & 72.6 & 66.3\\
    256 & 45.5 & 32.4 & 72.1 & 70.5 \\
    512 & 42.9 & 34.0 & 71.4 & 68.5 \\
    1024 & 47.0 & 32.1 & 73.8 & 65.9 \\
    2048 & 46.2 & 32.7 & 72.2 & 73.8 \\
    4096 & 48.0 & 29.8 & 72.6 & 73.0 \\
    8192 & 48.6 & 30.3 & 73.9 & 72.4\\
    16384 & 48.9 & 30.1 & 74.4 & 70.0\\
    32768 & 49.4 & 30.8 & 75.9 & 74.4 \\
    \bottomrule
  \end{tabular}
  \end{sc}
}

\subfloat[.48\textwidth][Rotated SVHN (3 Layers)]{
  \begin{sc}
  \begin{tabular}{ccccc}
    \toprule
    Width & AA & AF & LA & JA \\
    \midrule
    32 & 29.7 & 42.5 & 68.3 & 63.4 \\ 
    64 & 32.1 & 44.5 & 72.0 & 70.5 \\
    128 & 34.1 & 44.3 & 75.7 & 74.9 \\
    256 & 35.9 & 43.9 & 75.5 & 77.6 \\
    512 & 36.9 & 44.2 & 76.3 & 77.5 \\
    1024 & 37.0 & 44.5 & 77.1 & 79.8 \\
    2048 & 37.1 & 44.7 & 77.9 & 79.8 \\
    4096 & 37.8 & 43.1 & 77.8 & 80.0 \\
    8192 & 38.5 & 44.4 & 78.3 & 80.3 \\
    16384 & 38.7 & 44.0 & 79.7 & 80.5 \\
    32768 & 39.7 & 43.0 & 80.0 & 81.7\\
    \bottomrule
  \end{tabular}
  \end{sc}
}
\subfloat[.48\textwidth][Rotated GTSRB (3 Layers)]{
\begin{sc}
  \begin{tabular}{ccccc}
    \toprule
    Width & AA & AF & LA & JA \\
    \midrule
    32 & 39.9 & 33.8 & 66.6 & 67.4\\
    64 & 40.6 & 34.5 & 68.2 & 67.5 \\
    128 & 43.0 & 32.8 & 69.5 & 68.9 \\
    256 & 44.1 & 31.8 & 69.6 & 69.3 \\
    512 & 43.0 & 33.9 & 69.4 & 73.3 \\
    1024 & 41.3 & 36.9 & 71.4 & 72.0 \\
    2048 & 45.2 & 32.0 & 70.7 & 73.0 \\
    4096 & 44.3 & 33.7 & 72.5 & 71.4  \\
    8192 & 44.7 & 34.6 & 72.9 & 72.3 \\
    16384 & 43.1 & 35.6 & 73.5 & 74.8 \\
    32768 & 46.6 & 33.9 & 75.8 & 72.8 \\
    \bottomrule
  \end{tabular}
  \end{sc}
}
\caption{We report the numbers from our experiments on more layers on SVHN and GTRSB. We see that the trend of diminishing returns holds. We also see that increasing depth often leads to an increase in forgetting.}
\label{fig:svhngstrb}
\end{figure*}

\subsection{Effect of Dropout on Distance from Initialization}
We explore how Dropout affects the distance from initialization. Dropout is also a common technique employed to enforce sparsity in a trained network. To test the connection between Dropout values and distance from initialization, we train networks of varying widths with Dropout probabilities in $0.1$, $0.2$ and $0.3$. We plot the results of these experiments in \Cref{fig:dropoutdistance}. We see that across different Dropout probabilities, the connection between Distance from Initialization and network width holds. We also add the predicted curve formed from \Cref{ass:distance}. 
\begin{figure}
    \centering
    \begin{subfigure}[b]{0.48\textwidth}
        \centering
        \includegraphics[width=\textwidth]{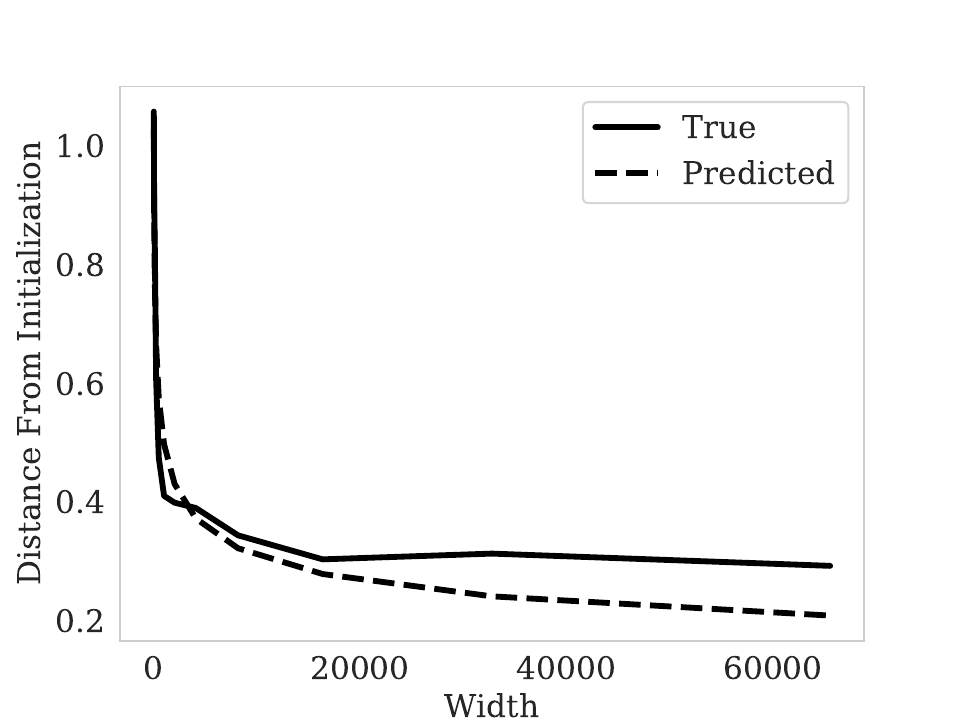}
        \caption{$.1$ Dropout Probability}
        \label{fig:sub1}
    \end{subfigure}
    \hfill
    \begin{subfigure}[b]{0.48\textwidth}
        \centering
        \includegraphics[width=\textwidth]{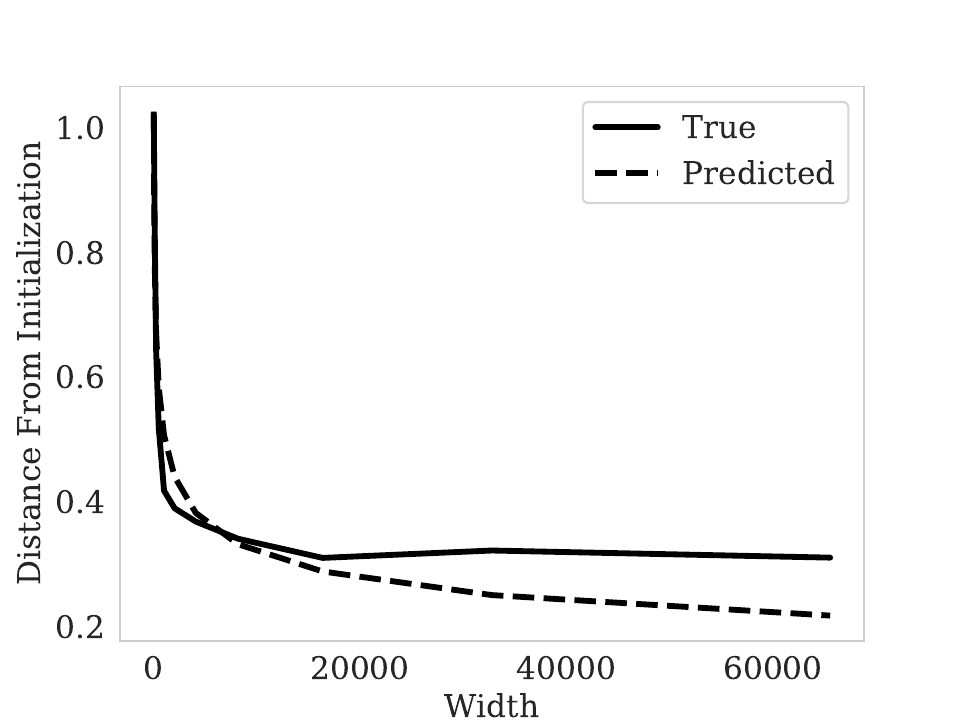}
        \caption{$.2$ Dropout Probability}
        \label{fig:sub2}
    \end{subfigure}
    \\
    \begin{subfigure}[b]{0.48\textwidth}
        \centering
        \includegraphics[width=\textwidth]{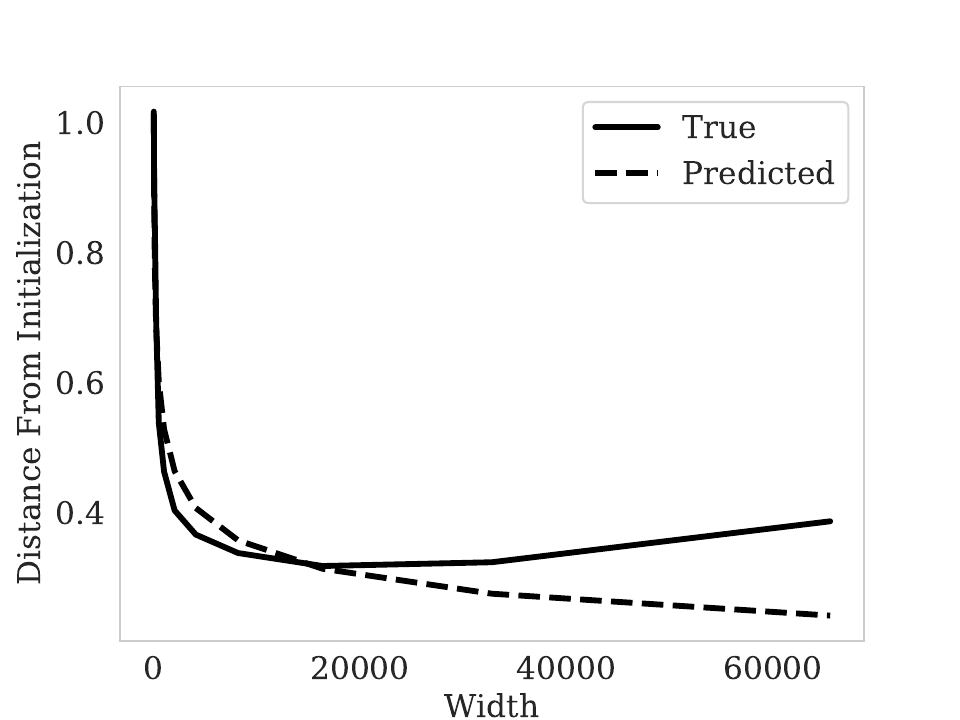}
        \caption{$.3$ Dropout Probability}
        \label{fig:sub3}
    \end{subfigure}
    \caption{Distance from Initialization  for different Dropout probabilities.}
    \label{fig:dropoutdistance}
\end{figure}

\subsection{Results at larger depths}
As an extension to the experiments in the main body of this paper, we run our experiments at larger depths. Specifically, we train networks of depths $4$, $5$, and $6$. We plot our results in \Cref{fig:moredepths}. Across most datasets and widths, we see that diminishing returns in terms of forgetting occurs as width increases. We highlight that at $6$ hidden layers, we see that the results are relatively noisy. This may be explained by the vanishing gradients that occurs at this depth. 

\begin{figure}
    \centering
    \begin{subfigure}[b]{0.48\textwidth}
        \centering
        \includegraphics[width=\textwidth]{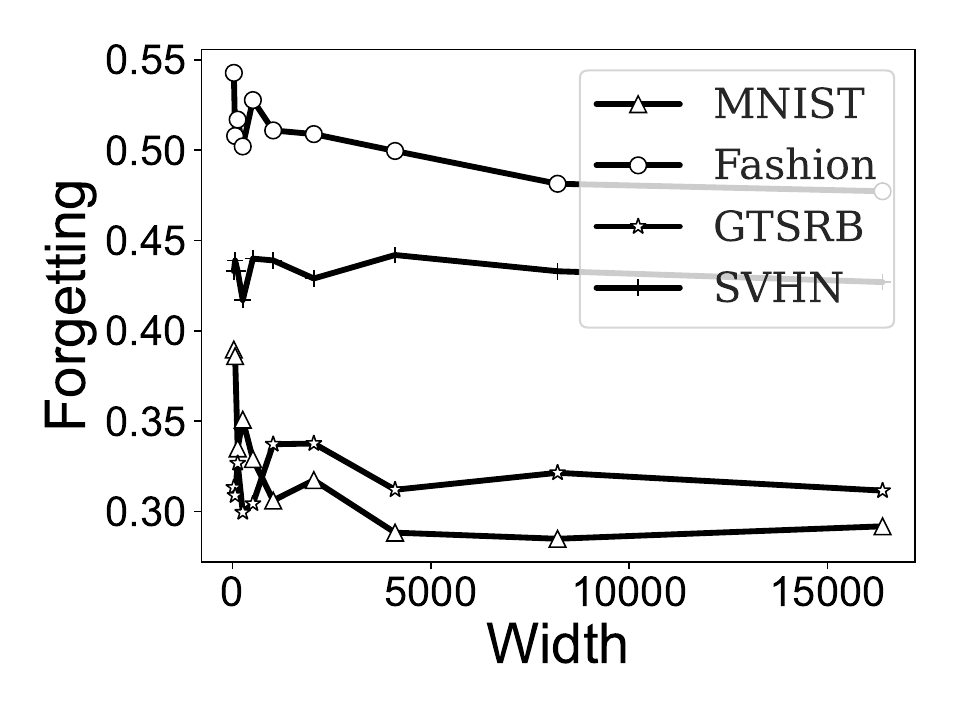}
        \caption{$4$ Hidden Layers}
        \label{fig:sub1}
    \end{subfigure}
    \hfill
    \begin{subfigure}[b]{0.48\textwidth}
        \centering
        \includegraphics[width=\textwidth]{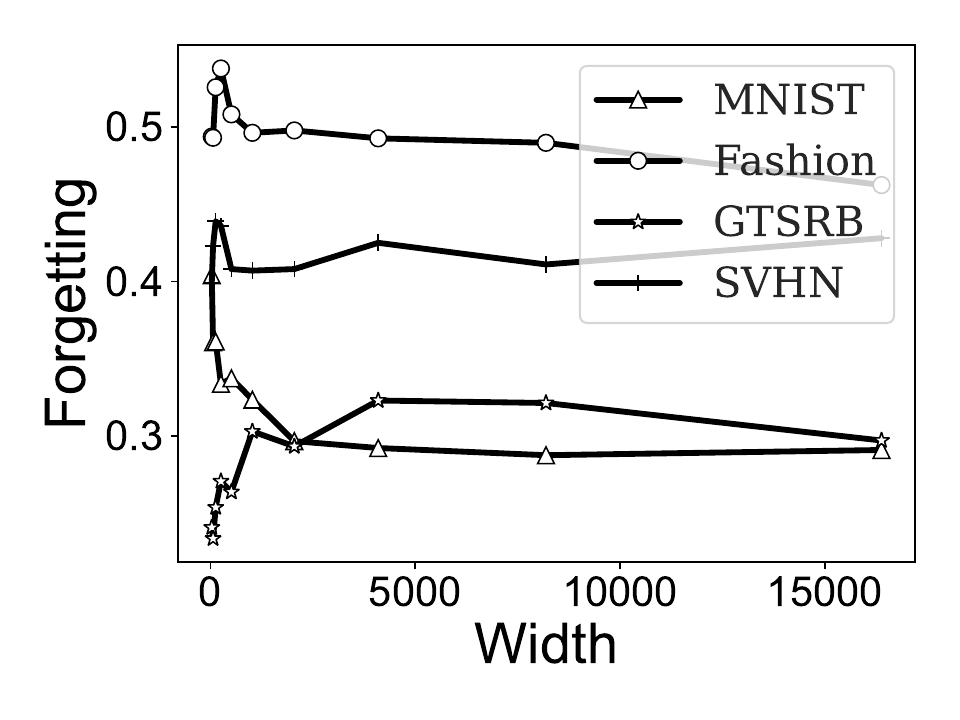}
        \caption{$5$ Hidden Layers}
        \label{fig:sub2}
    \end{subfigure}
    \\
    \begin{subfigure}[b]{0.48\textwidth}
        \centering
        \includegraphics[width=\textwidth]{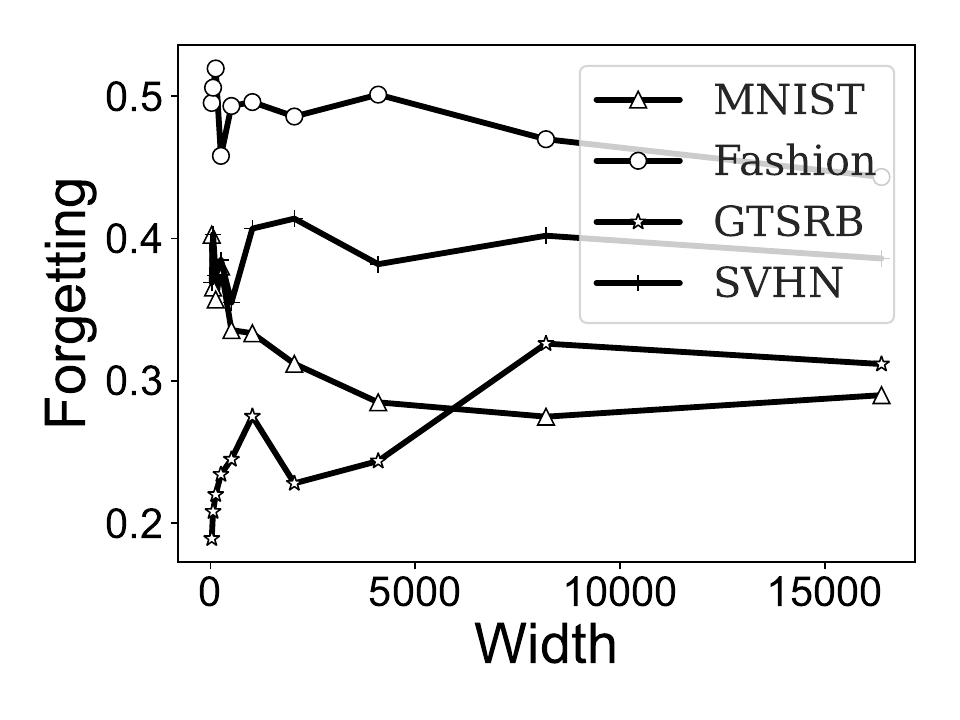}
        \caption{$6$ Hidden Layers}
        \label{fig:sub3}
    \end{subfigure}
    \caption{We plot forgetting as width of the network is increased for different number of hidden layers. }
    \label{fig:moredepths}
\end{figure}

\begin{figure}
    \centering
    \includegraphics[scale=.8]{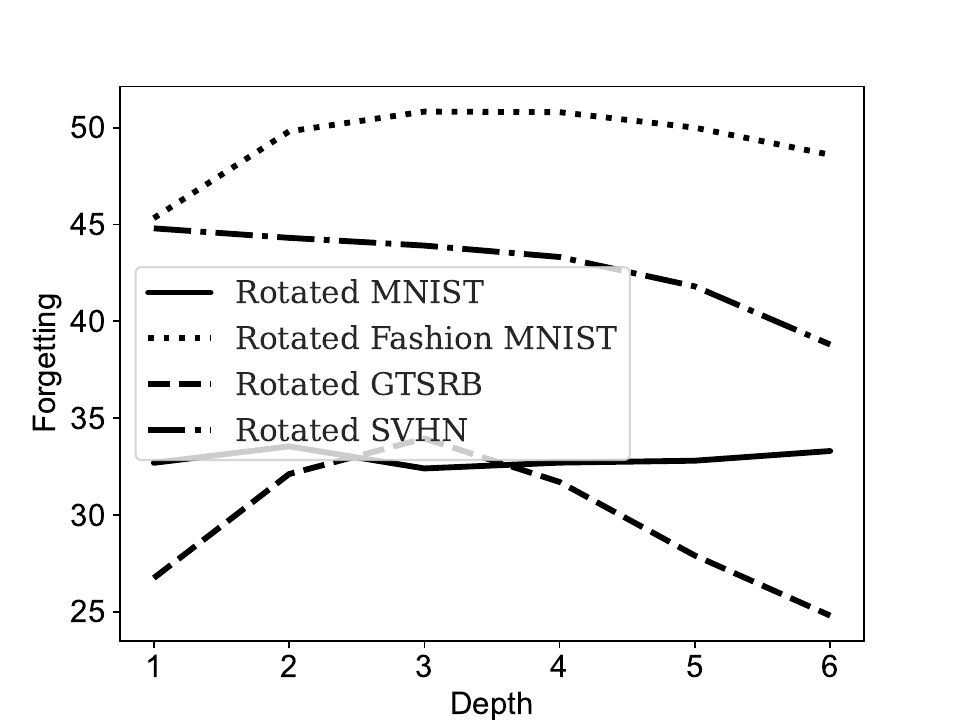}
    \caption{Average Forgetting as Depth is increased. We see that as depth is increased, forgetting increases. However, as depth is increased further, the accuracy goes down due to vanishing gradients. This artificially causes the forgetting to decrease.}
    \label{fig:moredepthsaverage}
\end{figure}

\section{Theoretical Analysis}
We will first prove the main claim connecting width and continual learning \Cref{lem:ourperturbation}
in \Cref{sec:theorem51}. We then extend this analysis to noise stability in \Cref{sec:noise_stabillity}.

\subsection{Proof of \Cref{lem:ourperturbation}}
\label{sec:theorem51}
We formalize the proof of \Cref{lem:ourperturbation} given the intuition from proof sketch of from the main body.  To do the analysis for $\epsilon_{t, t+1}$, we split the proof into three parts: (1) finding how many active rows are shared between anything two layers, (2) finding how far these active rows can change during training, and (3) combining the two parts together using perturbation analysis.

\intersection*
\begin{proof}
    Let $\mathbb{I}_i$ be the indicator random variable of whether a row $i$ is active in both models $\mathbf{M}_t$ and $\mathbf{M}_{t+1}$, i.e. $\mathbb{I}_i$ if $i \in \mathcal{A}_{t, l} \land i \in  \mathcal{A}_{t^{\prime}, l}$ and $0$ otherwise. Therefore, the expected value of the size of the intersection of the two sets $\mathcal{A}_{t, l}$ and $\mathcal{A}_{t^{\prime}}$ is 

    \begin{align}
        \mathbb{E}(|\mathcal{A}_{t, l} \cap \mathcal{A}_{t^{\prime}, l}|) &=\mathbb{E}\left[\sum_{i \in [W]} \mathbb{I}_i\right]\nonumber\\
        &=\sum_{i \in [W]}\mathbb{E}\left[ \mathbb{I}_i\right]\nonumber
    \end{align}
    Now, given that the probability that each row $i$ is in a given active set with probability $\alpha$, the probability a row is randomly in both active sets is $\alpha^2$, i.e. $\mathbb{E}(\mathbb{I}_i) = \alpha^2$. Given there are $W$ total rows this yields
    $\mathbb{E}(|\mathcal{A}_{t, l} \cap \mathcal{A}_{t^{\prime}, l}|) = \alpha^2 W$. 
\end{proof}

\perturbanytask*
\begin{proof}
    
    Now, in expectation, the size of $\mathcal{A}_{t, l}$ is $\alpha W$, i.e.
    $$\mathbb{E}(|\mathcal{A}_{t, l}|) = \alpha W\text{.}$$
    Only $\alpha^2W$ of the $\alpha W$ active rows $\mathcal{A}_{t, l}$ will intersect $\mathcal{A}_{t + 1, l}$ in expectation, from \Cref{lem:intersection}. Therefore, when training $\mathbf{M}_{t+1}$ only $\alpha^2W$ of the $\alpha W$ rows in expectation will change from its initialization. The rest of the rows will stay unchanged during training for task $t+1$.
    Now, $\mathbf{A}_{t, l}[\mathcal{A}_{t, l}](x) - \mathbf{A}_{t+1, l}[\mathcal{A}_{t, l}]$ is a matrix in $\mathbb{R}^{\mathcal{A}_{t, l} \times W}$. Moreover, this matrix will have $\alpha(1 - \alpha)W$ rows of all $0$'s in expectation. For the other rows, we know from \Cref{ass:distance}, the difference of the two layers indexed by $\mathcal{A}_{t+1, l}$ is bounded by 
    \begin{align}
       \frac{\|\mathbf{A}_{l, t+1}[\mathcal{A}_{l, t+1}] - \mathbf{A}_{l, t}[\mathcal{A}_{l, t+1}]\|_F}{\|\mathbf{A}_{l, t}[\mathcal{A}_{l, t+1}]\|_2} &\leq \gamma[\alpha W]^{-\beta}\nonumber 
    \end{align}
    By the definition of Frobenius Norm
    \begin{align}
        \sum_{i \in \mathcal{A}_{l, t+1}} \|\mathbf{A}_{l, t+1}[i] - \mathbf{A}_{l, t}[i]\|_2^2 &= \|\mathbf{A}_{l, t+1}[\mathcal{A}_{l, t+1}] - \mathbf{A}_{l, t}[\mathcal{A}_{l, t+1}]\|_F^2 \nonumber\\
        &\leq \gamma^2[\alpha W]^{-2\beta}\|\mathbf{A}_{l, t}[\mathcal{A}_{l, t+1}]\|_2^2\nonumber 
    \end{align}
    
    Since the expectation of sum is the sum of expectation, we have that 
$$\mathbb{E}\left(\sum_{i \in \mathcal{A}_{l, t+1}} \|\mathbf{A}_{l, t+1}[i] - \mathbf{A}_{l, t}[i]\|_2^2\right) = \sum_{i \in \mathcal{A}_{l, t+1}} \mathbb{E}\left(\|\mathbf{A}_{l, t+1}[i] - \mathbf{A}_{l, t}[i]\|_2^2\right)\text{.}$$ Since all rows are exchangeable under training, we have that the expected $\ell_2$ norm of a row is upper bounded by
$$\mathbb{E}\left[\|\mathbf{A}_{l, t+1}[i] - \mathbf{A}_{l, t}[i]\|_2^2\right] \leq \gamma^2[\alpha W ]^{-2\beta - 1}\|\mathbf{A}_{l, t}[\mathcal{A}_{l, t+1}]\|_2^2\text{.}$$
Therefore, for rows in both active sets $\mathcal{A}_{l, t+1} \cup \mathcal{A}_{l, t}$, the expected $\ell_2$ norm of the difference of the rows in $\mathbf{M}_t$ and $\mathbf{M}_{t+1}$ is $\gamma^2[\alpha W ]^{-2\beta - 1}$. For notational ease, let $\mathcal{I} =  \mathcal{A}_{t, l} \cap \mathcal{A}_{t+1, l}$ denote the intersection of the active rows while $\mathcal{O} = \mathcal{A}_{t, l} \cap \mathcal{A}_{t+1, l}^C$ denote the active rows not in the active task.  Therefore, 

    \begin{align}
        \mathbb{E}&\left[\|\mathbf{A}_{t, l}[\mathcal{A}_{t, l}] - \mathbf{A}_{t+1, l}[\mathcal{A}_{t, l}]\|_2\right] \nonumber\\
        &\leq \mathbb{E}\left[\left\|\mathbf{A}_{t, l}[\mathcal{I}] - \mathbf{A}_{t+1, l}[\mathcal{I}]\right \|_2  +  \left\|\mathbf{A}_{t, l}[\mathcal{O}] - \mathbf{A}_{t+1, l}[\mathcal{O}]\right \|_2 \right]\label{eq:57}\\
        &= \mathbb{E}\left[ \left\|\mathbf{A}_{t, l}[\mathcal{I}] - \mathbf{A}_{t+1, l}[\mathcal{I}]\right \|_2\right]\label{eq:58}\\
        &\leq \mathbb{E}\left[ \left\|\mathbf{A}_{t, l}[\mathcal{I}] - \mathbf{A}_{t+1, l}[\mathcal{I}]\right \|_F\right]\nonumber\\
        &\leq \sqrt{\sum_{i \in \mathcal{I}} \mathbb{E}\left[ \left\|\mathbf{A}_{t, l}[i] - \mathbf{A}_{t+1, l}[i]\right \|_2^2\right]}\label{eq:61}\\
        &\leq \sqrt{\alpha^2W \cdot  \gamma^2[\alpha W ]^{-2\beta - 1}\|\mathbf{A}_{l, t}[\mathcal{A}_{l, t+1}]\|_2^2} \nonumber\\
        &\leq \gamma W^{-\beta} \alpha^{\frac{1-2\beta}{2}}\|\mathbf{A}_{l, t}[\mathcal{A}_{l, t+1}]\|_2 \nonumber\\
         &= \gamma W^{-\beta} \alpha^{\frac{1-2\beta}{2}}\|\mathbf{A}_{l, t}[\mathcal{A}_{l, t}]\|_2 \lambda_{t, t+1}^l\nonumber
    \end{align}
    
Here, \Cref{eq:57} comes from splitting the matrix-multiply by different rows, \Cref{eq:58} comes from seeing that all rows in $\mathcal{O}$ will remain unchanged after training, and \Cref{eq:61} comes from $\mathbb{E}(\sqrt{X}) \leq \sqrt{\mathbb{E}(X)}$ by Jensen's Inequality for a random variable $X$.

From here, we have that
$$\mathbb{E} \left[\frac{\|\mathbf{A}_{t, l}[\mathcal{A}_{t, l}] - \mathbf{A}_{t+1, l}[\mathcal{A}_{t, l}]\|_2}{\|\mathbf{A}_{l, t}[\mathcal{A}_{l, t}]\|_2 } \right]\leq  \gamma W^{-\beta} \alpha^{\frac{1-2\beta}{2}}\lambda_{t, t+1}^l$$

\end{proof}
\begin{restatable}{lemma}{singleperturb}
    \label{lem:singleperturb}
        Say we generate a series of models $\mathbf{M}_1, \dots, \mathbf{M}_T$ by training sequentially on datasets $\mathcal{D}_1, \dots, \mathcal{D}_T$.  Let $\lambda_{i, j}^l = \frac{\|\mathbf{A}_{l, j}[\mathcal{A}_{l, i}]\|_2}{\|\mathbf{A}_{l, i}[\mathcal{A}_{l, i}]\|_2}$ denote the ratio of the spectral norms of the weights of different row indeces for different tasks.  Moreover, let $\bar{\lambda} = \underset{l \in [L], i, j \in [T]}{\max}\lambda_{i, j}^l$.  For any input vector from the $i$th dataset $x \in \mathcal{D}_i$, the $\ell_2$ norm of the difference of the outputs from models $\mathbf{M}_t$ and $\mathbf{M}_{t+1}$ are upper bounded by 
    $$\forall x \in \mathcal{D}_t, \quad \mathbb{E}\left[\|\mathbf{M}_t(x) - \mathbf{M}_{t+1}(x) \|_2 \right]\leq L2^L\bar{\lambda}\chi \left(  \prod_{l=1}^L L_l\|\mathbf{A}_{t, l}\|_2 \right)  \gamma W^{-\beta} \alpha^{\frac{1-2\beta}{2}}\text{.}$$ 
\end{restatable}
\begin{proof}
    We will begin by proving the perturbation analysis
      The first part of this proof mainly follows from \citet{Neyshabur2017}. We restate it here with the differing notation for clarity and completeness. We will prove the induction hypothesis that for any $x \in \mathcal{D}_t$,
      $$\|\mathbf{M}_{t, l}(x) -\mathbf{M}_{t+1, l}(x)\|_2 \leq 2^l \|x\|_2 \left(  \prod_{i=1}^l L_i\|\mathbf{A}_{t, l}[\mathcal{A}_{t, l}]\|_2 \right) \sum_{i=1}^{l}\frac{\|\mathbf{A}_{t, l}[\mathcal{A}_{t, l}] - \mathbf{A}_{t+1, l}[\mathcal{A}_{t, l}]\|_2}{\|\mathbf{A}_{t, l}[\mathcal{A}_{t, l}]\|_2}\text{.}$$ Here, $\mathbf{M}_{t, l}$ and $\mathbf{M}_{t+1, l}$ denote the models $\mathbf{M}_{t}$ and $\mathbf{M}_{t^\prime}$ respectively with only the first $l$ layers. The base case of induction trivially holds, given that $\|\mathbf{M}_{t, 0}(x) -\mathbf{M}_{t+1, 0}(x)\|_2 = 0$ by definition. Now, we prove the induction step. Assume that the induction hypothesis holds for $l-1$. We will prove that it holds for $l$.  For notational ease, denote $\mathbf{U}_{t, l} = \mathbf{A}_{t, l}[\mathcal{A}_{t, l}] - \mathbf{A}_{t+1, l}[\mathcal{A}_{t, l}]$, $x_{t,l} = \mathbf{M}_{t, l}(x)$, and $x_{t+1,l} = \mathbf{M}_{t+1, l}(x)$.  We have that
    \begin{align}
        \|x_{t, l} &- x_{t+1, l}\|_2 \nonumber\\
        &\leq \|\left(\mathbf{A}_{t, l}[\mathcal{A}_{t,l}] + \mathbf{U}_{t, l}\right)\phi_l(x_{t+1, l-1}) - \mathbf{A}_{t, l}[\mathcal{A}_{t,l}]\phi_l(x_{t, l-1})\|_2\nonumber\\
        &\leq \|\left(\mathbf{A}_{t, l}[\mathcal{A}_{t,l}] + \mathbf{U}_{t, l}\right)(\phi_l(x_{t+1, l-1}) - \phi_l(x_{t, l-1})) + \mathbf{U}_{t, l}\phi_l(x_{t, l-1})\|_2\nonumber\\
        &\leq \left(\|\mathbf{A}_{t, l}[\mathcal{A}_{t,l}]\|_2 + \|\mathbf{U}_{t, l}\|_2\right)\|\phi_l(x_{t+1, l-1}) - \phi_l(x_{t, l-1})\|_2 + \|\mathbf{U}_{t, l}\|_2\|\phi_l(x_{t, l-1})\|_2\nonumber\\
        &\leq L_l  \left(\|\mathbf{A}_{t, l}[\mathcal{A}_{t,l}]\|_2 + \|\mathbf{U}_{t, l}\|_2\right)\|x_{t+1, l-1} - x_{t, l-1}\|_2 + L_l \|\mathbf{U}_{t, l}\|_2\|x_{t, l-1}\|_2 \label{eq:ney2}\\
        &\leq 2L_l\left(\|\mathbf{A}_{t, l}[\mathcal{A}_{t,l}]\|_2\right)\|x_{t+1, l-1} - x_{t, l-1}\|_2 + L_l \|\mathbf{U}_{t, l}\|_2\|x_{t, l-1}\|_2\nonumber\\
        &\leq 2L_l \left(\|\mathbf{A}_{t, l}[\mathcal{A}_{t,l}]\|_2\right)\left(1 + \frac{1}{L}\right)^{l-1} \|x_{t, 0}\|_2 \left(  \prod_{i=1}^{l-1} L_i \|\mathbf{A}_{t, i}[\mathcal{A}_{t,i}]\|_2 \right) \sum_{i=1}^{l-1}\frac{\|\mathbf{U}_{t, l}\|_2}{\|\mathbf{A}_{t, i}[\mathcal{A}_{t,i}]\|_2} \nonumber \\
        & \quad \quad \quad \quad \quad \quad \quad \quad \quad \quad  \quad \quad \quad \quad \quad + L_l \|\mathbf{U}_{t, l}\|_2\|x_{t, l-1}\|_2 \label{eq:ney1}\\
        &\leq 2^lL_l\left(  \prod_{i=1}^{l-1} L_i \|\mathbf{A}_{t, i}[\mathcal{A}_{t,i}]\|_2 \right) \sum_{i=1}^{l-1}\frac{\|\mathbf{U}_{t, i}\|_2}{\|\mathbf{A}_{t, i}[\mathcal{A}_{t,i}]\|_2}\|x_{t, 0}\|_2  \nonumber\\
        &\quad \quad \quad \quad \quad \quad \quad \quad \quad \quad  \quad \quad \quad \quad \quad  + L_l \|x_{t, 0}\|_2 \|\mathbf{U}_{t, l}\|_2\prod_{i=1}^{l-1}L_{i}\|\mathbf{A}_{t, i}[\mathcal{A}_{t,i}]\|_2\nonumber\\
        &\leq 2^lL_l\left(  \prod_{i=1}^{l-1} L_i  \|\mathbf{A}_{t, i}[\mathcal{A}_{t,i}]\|_2 \right) \sum_{i=1}^{l-1}\frac{\|\mathbf{U}_{t, l}\|_2}{\|\mathbf{A}_{t, i}[\mathcal{A}_{t,i}]\|_2}\|x_{t, 0}\|_2 \nonumber\\
        &\quad \quad \quad \quad \quad \quad \quad \quad \quad \quad  \quad \quad \quad \quad \quad + \|x_{t, 0}\|_2 \frac{\|\mathbf{U}_{t, l}\|_2}{\|\mathbf{A}_{t, l}[\mathcal{A}_{t,l}]\|_2}\prod_{i=1}^{l}L_{i}\|\mathbf{A}_{t, i}[\mathcal{A}_{t,i}]\|_2\nonumber\\
        &\leq 2^L\left(  \prod_{i=1}^{l} L_i \|\mathbf{A}_{t, i}[\mathcal{A}_{t,i}]\|_2 \right) \sum_{i=1}^{l}\frac{\|\mathbf{U}_{t, l}\|_2}{\|\mathbf{A}_{t, i}[\mathcal{A}_{t,i}]\|_2}\|x_{t, 0}\|_2 \nonumber
    \end{align}
    Here, \Cref{eq:ney2} comes from the fact that $\phi_l$ is $L_l$-Lipschitz smooth and that $\phi_l(0) = 0$. Moreover, \Cref{eq:ney1} comes from applying the induction hypothesis. Therefore, we have proven the induction hypothesis for all layers.  We have 
     $$\forall x \in \mathcal{D}_t, \quad \|\mathbf{M}_t(x) - \mathbf{M}_{t+1}(x) \|_2 \leq 2^L\|x\|_2 \left(  \prod_{l=1}^L L_l\|\mathbf{A}_{t, l}[\mathcal{A}_{t,l}]\|_2 \right) \sum_{l=1}^{L}\frac{\|\mathbf{U}_{t, l}\|_2}{\|\mathbf{A}_{t, l}[\mathcal{A}_{t,l}]\|_2}\text{.}$$ 
     Denoting $\chi = \max_{x \in \mathcal{D}_t} \|x\|_2$ and using \Cref{lem:perturbanytask}
      $$\forall x \in \mathcal{D}_t, \quad\mathbb{E}\left[ \|\mathbf{M}_t(x) - \mathbf{M}_{t+1}(x) \|_2\right] \leq 2^L\chi \left(  \prod_{l=1}^L L_l\|\mathbf{A}_{t, l}[\mathcal{A}_{t,l}]\|_2 \right) \sum_{l=1}^{L}\lambda_{t, t+1}^l \gamma W^{-\beta} \alpha^{\frac{1-2\beta}{2}}\text{.}$$ 
      Here, we reminded the reader that $\lambda_{i, j}^l = \frac{\|\mathbf{A}_{l, i}[\mathcal{A}_{l, i}]\|_2}{\|\mathbf{A}_{l, i}[\mathcal{A}_{l, j}]\|_2}$ denotes the ratio of the spectral norms of the weights of different row indeces for different tasks.  Moreover, for any matrix, removing rows cannot increase the matrix norm, we have that $\|\mathbf{A}_{t,l}\|_2 \geq \|\mathbf{A}_{t,l}[\mathcal{A}_{t,l}\|_2$. 
      Therefore, we have $$\prod_{l=1}^L \|\mathbf{A}_{t, l}[\mathcal{A}_{t,l}]\|_2 \leq \prod_{l=1}^L \|\mathbf{A}_{t, l}\|_2\text{.}$$ Given $\bar{\lambda} = \underset{l \in [L], i, j \in [T]}{\max}\lambda_{i, j}^l$,
      $$\forall x \in \mathcal{D}_t, \quad \mathbb{E}\left[\|\mathbf{M}_t(x) - \mathbf{M}_{t+1}(x) \|_2 \right]\leq L2^L\bar{\lambda}\chi \left(  \prod_{l=1}^L L_l\|\mathbf{A}_{t, l}\|_2 \right)  \gamma W^{-\beta} \alpha^{\frac{1-2\beta}{2}}\text{.}$$ 

\end{proof}

\ourperturbation*
\begin{proof}
    We need only repeat the proof above but with a different perturbation to account for the number of tasks. We repeat it for clarity.
   We will prove the induction hypothesis that for any $x \in \mathcal{D}_t$,
      $$\|\mathbf{M}_{t, l}(x) -\mathbf{M}_{t^{\prime}, l}(x)\|_2 \leq 2^l \|x\|_2 \left(  \prod_{i=1}^l L_i\|\mathbf{A}_{t, l}[\mathcal{A}_{t, l}]\|_2 \right) \sum_{i=1}^{l}\frac{\|\mathbf{A}_{t, l}[\mathcal{A}_{t, l}] - \mathbf{A}_{t^{\prime}, l}[\mathcal{A}_{t, l}]\|_2}{\|\mathbf{A}_{t, l}[\mathcal{A}_{t, l}]\|_2}\text{.}$$ Here, $\mathbf{M}_{t, l}$ and $\mathbf{M}_{t^{\prime}, l}$ denote the models $\mathbf{M}_{t}$ and $\mathbf{M}_{t^\prime}$ respectively with only the first $l$ layers. The base case of induction trivially holds, given that $\|\mathbf{M}_{t, 0}(x) -\mathbf{M}_{t^{\prime}, 0}(x)\|_2 = 0$ by definition. Now, we prove the induction step. Assume that the induction hypothesis holds for $l-1$. We will prove that it holds for $l$.  For notational ease, denote $\mathbf{U}_{t, l} = \mathbf{A}_{t, l}[\mathcal{A}_{t, l}] - \mathbf{A}_{t^{\prime}, l}[\mathcal{A}_{t, l}]$, $x_{t,l} = \mathbf{M}_{t, l}(x)$, and $x_{t^{\prime},l} = \mathbf{M}_{t^{\prime}, l}(x)$.  We have that
    \begin{align}
        \|x_{t, l} &- x_{t^{\prime}, l}\|_2 \nonumber\\
        &\leq \|\left(\mathbf{A}_{t, l}[\mathcal{A}_{t,l}] + \mathbf{U}_{t, l}\right)\phi_l(x_{t^{\prime}, l-1}) - \mathbf{A}_{t, l}[\mathcal{A}_{t,l}]\phi_l(x_{t, l-1})\|_2\nonumber\\
        &\leq \|\left(\mathbf{A}_{t, l}[\mathcal{A}_{t,l}] + \mathbf{U}_{t, l}\right)(\phi_l(x_{t^{\prime}, l-1}) - \phi_l(x_{t, l-1})) + \mathbf{U}_{t, l}\phi_l(x_{t, l-1})\|_2\nonumber\\
        &\leq \left(\|\mathbf{A}_{t, l}[\mathcal{A}_{t,l}]\|_2 + \|\mathbf{U}_{t, l}\|_2\right)\|\phi_l(x_{t^{\prime}, l-1}) - \phi_l(x_{t, l-1})\|_2 + \|\mathbf{U}_{t, l}\|_2\|\phi_l(x_{t, l-1})\|_2\nonumber\\
        &\leq L_l  \left(\|\mathbf{A}_{t, l}[\mathcal{A}_{t,l}]\|_2 + \|\mathbf{U}_{t, l}\|_2\right)\|x_{t^{\prime}, l-1} - x_{t, l-1}\|_2 + L_l \|\mathbf{U}_{t, l}\|_2\|x_{t, l-1}\|_2 \label{eq:ney2ag}\\
        &\leq 2L_l\left(\|\mathbf{A}_{t, l}[\mathcal{A}_{t,l}]\|_2\right)\|x_{t^{\prime}, l-1} - x_{t, l-1}\|_2 + L_l \|\mathbf{U}_{t, l}\|_2\|x_{t, l-1}\|_2\nonumber\\
        &\leq 2L_l \left(\|\mathbf{A}_{t, l}[\mathcal{A}_{t,l}]\|_2\right)\left(1 + \frac{1}{L}\right)^{l-1} \|x_{t, 0}\|_2 \left(  \prod_{i=1}^{l-1} L_i \|\mathbf{A}_{t, i}[\mathcal{A}_{t,i}]\|_2 \right) \sum_{i=1}^{l-1}\frac{\|\mathbf{U}_{t, l}\|_2}{\|\mathbf{A}_{t, i}[\mathcal{A}_{t,i}]\|_2} \nonumber \\
        & \quad \quad \quad \quad \quad \quad \quad \quad \quad \quad  \quad \quad \quad \quad \quad + L_l \|\mathbf{U}_{t, l}\|_2\|x_{t, l-1}\|_2 \label{eq:ney1ag}\\
        &\leq 2^lL_l\left(  \prod_{i=1}^{l-1} L_i \|\mathbf{A}_{t, i}[\mathcal{A}_{t,i}]\|_2 \right) \sum_{i=1}^{l-1}\frac{\|\mathbf{U}_{t, i}\|_2}{\|\mathbf{A}_{t, i}[\mathcal{A}_{t,i}]\|_2}\|x_{t, 0}\|_2  \nonumber\\
        &\quad \quad \quad \quad \quad \quad \quad \quad \quad \quad  \quad \quad \quad \quad \quad  + L_l \|x_{t, 0}\|_2 \|\mathbf{U}_{t, l}\|_2\prod_{i=1}^{l-1}L_{i}\|\mathbf{A}_{t, i}[\mathcal{A}_{t,i}]\|_2\nonumber\\
        &\leq 2^lL_l\left(  \prod_{i=1}^{l-1} L_i  \|\mathbf{A}_{t, i}[\mathcal{A}_{t,i}]\|_2 \right) \sum_{i=1}^{l-1}\frac{\|\mathbf{U}_{t, l}\|_2}{\|\mathbf{A}_{t, i}[\mathcal{A}_{t,i}]\|_2}\|x_{t, 0}\|_2 \nonumber\\
        &\quad \quad \quad \quad \quad \quad \quad \quad \quad \quad  \quad \quad \quad \quad \quad + \|x_{t, 0}\|_2 \frac{\|\mathbf{U}_{t, l}\|_2}{\|\mathbf{A}_{t, l}[\mathcal{A}_{t,l}]\|_2}\prod_{i=1}^{l}L_{i}\|\mathbf{A}_{t, i}[\mathcal{A}_{t,i}]\|_2\nonumber\\
        &\leq 2^L\left(  \prod_{i=1}^{l} L_i \|\mathbf{A}_{t, i}[\mathcal{A}_{t,i}]\|_2 \right) \sum_{i=1}^{l}\frac{\|\mathbf{U}_{t, l}\|_2}{\|\mathbf{A}_{t, i}[\mathcal{A}_{t,i}]\|_2}\|x_{t, 0}\|_2 \nonumber
    \end{align}
    Here, \Cref{eq:ney2ag} comes from the fact that $\phi_l$ is $L_l$-Lipschitz smooth and that $\phi_l(0) = 0$. Moreover, \Cref{eq:ney1ag} comes from applying the induction hypothesis. Therefore, we have proven the induction hypothesis for all layers.  We have 
     $$\forall x \in \mathcal{D}_t, \quad \|\mathbf{M}_t(x) - \mathbf{M}_{t^{\prime}}(x) \|_2 \leq 2^L\|x\|_2 \left(  \prod_{l=1}^L L_l\|\mathbf{A}_{t, l}[\mathcal{A}_{t,l}]\|_2 \right) \sum_{l=1}^{L}\frac{\|\mathbf{U}_{t, l}\|_2}{\|\mathbf{A}_{t, l}[\mathcal{A}_{t,l}]\|_2}\text{.}$$ 
     Now, 
     \begin{align}
        \frac{\|\mathbf{U}_{t, l}\|_2}{\|\mathbf{A}_{t, l}[\mathcal{A}_{t,l}]\|_2} &= \frac{\|\mathbf{A}_{t, l}[\mathcal{A}_{t, l}] - \mathbf{A}_{t^{\prime}, l}[\mathcal{A}_{t, l}]\|_2}{\|\mathbf{A}_{t, l}[\mathcal{A}_{t,l}]\|_2} \nonumber \\
        &=\frac{\left \|\sum_{i=t}^{t^\prime - 1} \mathbf{A}_{i, l}[\mathcal{A}_{i, l}] - \mathbf{A}_{i+1, l}[\mathcal{A}_{i+1, l}]\right\|_2}{\|\mathbf{A}_{t, l}[\mathcal{A}_{t,l}]\|_2} \nonumber \\
        &\leq\frac{\sum_{i=t}^{t^\prime - 1}\left \| \mathbf{A}_{i, l}[\mathcal{A}_{i, l}] - \mathbf{A}_{i+1, l}[\mathcal{A}_{i+1, l}]\right\|_2}{\|\mathbf{A}_{t, l}[\mathcal{A}_{t,l}]\|_2} \nonumber 
    \end{align}
    Therefore, in expectation,
    \begin{align}
       \mathbb{E}\left[\frac{\|\mathbf{U}_{t, l}\|_2}{\|\mathbf{A}_{t, l}[\mathcal{A}_{t,l}]\|_2} \right] &\leq(t^{\prime} - t)\bar{\lambda}\gamma W^{-\beta} \alpha^{\frac{1-2\beta}{2}}\nonumber 
     \end{align}
     Denoting $\chi = \max_{x \in \mathcal{D}_t} \|x\|_2$ and using \Cref{lem:perturbanytask}
      $$\forall x \in \mathcal{D}_t, \quad \mathbb{E}\left[\|\mathbf{M}_t(x) - \mathbf{M}_{t^{\prime}}(x) \|_2\right] \leq L2^L\chi \left(  \prod_{l=1}^L L_l\|\mathbf{A}_{t, l}[\mathcal{A}_{t,l}]\|_2 \right)(t^{\prime} - t)\bar{\lambda}\gamma W^{-\beta} \alpha^{\frac{1-2\beta}{2}}\text{.}$$ 
      Here, we reminded the reader that $\lambda_{i, j}^l = \frac{\|\mathbf{A}_{l, i}[\mathcal{A}_{l, i}]\|_2}{\|\mathbf{A}_{l, i}[\mathcal{A}_{l, j}]\|_2}$ denotes the ratio of the spectral norms of the weights of different row indeces for different tasks.  Moreover, for any matrix, removing rows cannot increase the matrix norm, we have that $\|\mathbf{A}_{t,l}\|_2 \geq \|\mathbf{A}_{t,l}[\mathcal{A}_{t,l}\|_2$. 
      Therefore, we have $$\prod_{l=1}^L \|\mathbf{A}_{t, l}[\mathcal{A}_{t,l}]\|_2 \leq \prod_{l=1}^L \|\mathbf{A}_{t, l}\|_2\text{.}$$ Given $\bar{\lambda} = \underset{l \in [L], i, j \in [T]}{\max}\lambda_{i, j}^l$,
      $$\forall x \in \mathcal{D}_t, \quad \mathbb{E}\left[\|\mathbf{M}_t(x) - \mathbf{M}_{t^{\prime}}(x) \|_2\right] \leq (t^{\prime} - t)L2^L\bar{\lambda}\chi \left(  \prod_{l=1}^L L_l\|\mathbf{A}_{t, l}\|_2 \right)  \gamma W^{-\beta} \alpha^{\frac{1-2\beta}{2}}\text{.}$$ 
\end{proof}

\subsection{Noise Stability}
\label{sec:noise_stabillity}

\smoothenedproof*
\begin{proof}
    We will prove the induction hypothesis that $$\|\mathbf{M}_{t, l}(x) - \mathbf{M}_{t^{\prime}, l}(x)\|_2 \leq \epsilon_l \|\mathbf{M}_{t, l}(x)\|_2 \text{,}$$
    where $\epsilon_l = \prod_{i}^l a_i \left(\sum_{i=1}^{l} b_{i}\right)$ where $a_i = c_i\mu_{i, t}L_i + \frac{L_i\mu_{i,l}\|\mathbf{U}_{i, l}\|_2c_{l}}{\|\mathbf{A}_{i, l}[\mathcal{A}_{i, l}]\|_2}$ and $b_i = \frac{L_i\mu_{i,l}\|\mathbf{U}_{i, l}\|_2c_{l}}{\|\mathbf{A}_{i, l}[\mathcal{A}_{i, l}]\|_2}$.
    The base case trivially holds given $\|\mathbf{M}_{t, 0}(x) - \mathbf{M}_{t^{\prime}, 0}(x)\|_2 = 0$. Here, $\mathbf{M}_{t, l}$ and $\mathbf{M}_{t^{\prime}, l}$ denote the models $\mathbf{M}_{t}$ and $\mathbf{M}_{t^\prime}$ respectively with only the first $l$ layers. We now perform our induction.
\begin{align}
    \|\mathbf{M}_{t, l}(x) - \mathbf{M}_{t^{\prime}, l}(x)\|_2 &= \|(\mathbf{A}_{t, l}[\mathcal{A}_{t, l}] - \mathbf{U}_{t, l})\phi_{l}(x_{t^\prime, {l-1}}) - \mathbf{A}_{t, l}[\mathcal{A}_{t, l}]\phi_{l}(x_{t, {l-1}})\|_2 \nonumber\\
    &\leq \|(\mathbf{A}_{t, l}[\mathcal{A}_{t, l}](\phi_{l}(x_{t^\prime, {l-1}}) - \phi_{l}(x_{t, {l-1}})) - \mathbf{U}_{t, l}\phi_{l}(x_{t^\prime, {l-1}})\|_2 \nonumber\\
     &\leq \|\mathbf{A}_{t, l}[\mathcal{A}_{t, l}]\|_2\|\phi_{l}(x_{t^\prime, {l-1}}) - \phi_{l}(x_{t, {l-1}})\|_2 + \|\mathbf{U}_{t, l}\|_2\|\phi_{l}(x_{t^\prime, {l-1}})\|_2 \nonumber\\
     &\leq L_l\|\mathbf{A}_{t, l}[\mathcal{A}_{t, l}]\|_2\|x_{t^\prime, {l-1}} - x_{t, {l-1}}\|_2 + L_l\|\mathbf{U}_{t, l}\|_2\|x_{t^\prime, {l-1}}\|_2 \label{eq:lip}\\
     &\leq L_l\epsilon_{l-1}\|\mathbf{A}_{t, l}[\mathcal{A}_{t, l}]\|_2\|x_{t, {l-1}}\|_2 + (1 + \epsilon_{l-1})L_l\|\mathbf{U}_{t, l}\|_2\|x_{t, {l-1}}\|_2 \label{eq:induct}\\
     &\leq c_lL_l\epsilon_{l-1}\|\mathbf{A}_{t, l}[\mathcal{A}_{t, l}]\|_2\|\phi_l(x_{t, l-1})\|_2 + (1 + \epsilon_{l-1})L_l\|\mathbf{U}_{t, l}\|_2\|x_{t, l-1}\|_2 \label{eq:activation contraction}\\
     &\leq c_l\mu_{t, l}L_l\epsilon_{l-1}\|\mathbf{A}_{t, l}[\mathcal{A}_{t, l}]\phi_l(x_{t, l-1})\|_2 + (1 + \epsilon_{l-1})L_l\|\mathbf{U}_{t, l}\|_2\|x_{t, l-1}\|_2 \label{eq:layer cushion}\\
     &\leq c_l\mu_{t, l}L_l\epsilon_{l-1}\|x_{t, l}\|_2 + (1 + \epsilon_{l-1})L_l\|\mathbf{U}_{t, l}\|_2\|x_{t, l-1}\|_2 \nonumber
\end{align}
Here, \Cref{eq:lip} comes from the Lipschitz-Smoothness of the activation layers, \Cref{eq:induct} comes from applying our induction hypothesis from the previous layer, \Cref{eq:activation contraction} comes from applying the activation contraction from \Cref{def:activation contraction}, \Cref{eq:layer cushion} comes from applying the layer cushion from \Cref{def:layer cushion}.
We also note that the ratio outputs of two subsequential layers is bounded as in the following.
\begin{align}
    \frac{\|x_{t, l-1}\|_2 }{\|x_{t, l}\|_2} &\leq  \frac{c_{l}\|\phi_l(x_{t, l-1})\|_2 }{\|x_{t, l}\|_2} \label{eq:activation contraction again}\\
    &\leq  \frac{c_{l} \mu_{t, l} }{\|\mathbf{A}_{t, l}[\mathcal{A}_{t, l}]\|_2} \label{eq:layer cushion again} 
\end{align}
Here, \Cref{eq:activation contraction again} come from applying the activation contraction from \Cref{def:activation contraction} and \Cref{eq:layer cushion again} comes from applying the layer cushion from \Cref{def:layer cushion}.
Therefore, we have that 
\begin{align}
    \epsilon_l &\leq \frac{c_l\mu_{t, l}L_l\epsilon_{l-1}\|x_{t, l}\|_2 + (1 + \epsilon_{l-1})L_l\|\mathbf{U}_{t, l}\|_2\|x_{t, l-1}\|_2}{\|x_{t, l}\|_2} \nonumber \\
    &\leq c_l\mu_{t, l}L_l\epsilon_{l-1} + \frac{(1 + \epsilon_{l-1})L_l\|\mathbf{U}_{t, l}\|_2\|x_{t, l-1}\|_2}{\|x_{t, l}\|_2} \nonumber \\
    &\leq c_l\mu_{t, l}L_l\epsilon_{l-1} + \frac{(1 + \epsilon_{l-1})\mu_{t,l}L_l\|\mathbf{U}_{t, l}\|_2c_{l}}{\|\mathbf{A}_{t, l}[\mathcal{A}_{t, l}]\|_2} \nonumber 
\end{align}

For mathematical ease, denote $a_l = c_l\mu_{t, l}L_l + \frac{L_l\mu_{t,l}\|\mathbf{U}_{t, l}\|_2c_{l}}{\|\mathbf{A}_{t, l}[\mathcal{A}_{t, l}]\|_2}$ and $b_l = \frac{L_l\mu_{t,l}\|\mathbf{U}_{t, l}\|_2c_{l}}{\|\mathbf{A}_{t, l}[\mathcal{A}_{t, l}]\|_2}$. Thus, we have
\begin{align}
    \epsilon_l &\leq a_l\epsilon_{l-1} + b_{l} \nonumber \\
    &\leq \prod_{i}^l a_i \sum_{i=1}^{l - 1} b_{i} + b_{l} \nonumber \\
    &\leq \prod_{i}^l a_i \left(\sum_{i=1}^{l - 1} b_{i} + b_{l}\right) \nonumber \\
    &\leq \prod_{i}^l a_i \left(\sum_{i=1}^{l} b_{i} \right) \nonumber
\end{align}

We have thus proved our hypothesis. We will now simplify the bound. We have that 
\begin{align}
    \mathbb{E}\left[a_l\right] &= \mathbb{E}\left[c_l\mu_{t, l}L_l + \frac{L_l\mu_{t,l}\|\mathbf{U}_{t, l}\|_2c_{l}}{\|\mathbf{A}_{t, l}[\mathcal{A}_{t, l}]\|_2}\right]\nonumber\\
    &\leq  c_l\mu_{t, l}L_l + L_l\mu_{t,l}(t^{\prime} - t)\gamma\bar{\lambda} W^{-\beta} \alpha^{\frac{1-2\beta}{2}} \label{eq:271}\\
    &\leq  c_l\mu_{t, l}L_l + L_lc_l\mu_{t,l}(t^{\prime} - t)\gamma\bar{\lambda} \nonumber
\end{align}
Here, \Cref{eq:271} comes from \Cref{lem:ourperturbation}. Similarly, we can bound $b_l$ such that 
\begin{align}
     \mathbb{E}\left[b_l\right] &= \mathbb{E}\left[\frac{L_l\mu_{t,l}\|\mathbf{U}_{t, l}\|_2c_{l}}{\|\mathbf{A}_{t, l}[\mathcal{A}_{t, l}]\|_2}\right] \nonumber\\
     &\leq L_lc_l\mu_{t,l}(t^{\prime} - t)\gamma\bar{\lambda} W^{-\beta} \alpha^{\frac{1-2\beta}{2}}\nonumber
\end{align}
Therefore, putting it all together yields 
\begin{align}
    \mathbb{E}\left[\epsilon_l \right]\leq \left(\prod_{i=1}^l c_i\mu_{t,i}L_i + L_ic_i\mu_{t,i}(t^{\prime} - t)\gamma\bar{\lambda}\right) \left(\sum_{i=1}^l L_ic_i(t^{\prime} - t)\gamma\bar{\lambda} W^{-\beta} \alpha^{\frac{1-2\beta}{2}}\mu_{t,i}\right) \nonumber\\
    \leq (t^{\prime} - t)\gamma\bar{\lambda} W^{-\beta} \alpha^{\frac{1-2\beta}{2}} \left(\prod_{i=1}^l c_i\mu_{t,i}L_i + L_i(t^{\prime} - t)\gamma\bar{\lambda}\mu_{t,i} \right)\left(\sum_{i=1}^l L_i\mu_{t,i}\right) \nonumber
\end{align}

We now have that $\forall x \in \mathcal{D}_t$,
$$\mathbb{E}\left[\|\mathbf{M}_{t}(x) - \mathbf{M}_{t^{\prime}}(x)\|_2\right]\leq \underset{x \in \mathcal{D}_t}{\max}\|\mathbf{M}_t(x)\|_2 \cdot (t^{\prime} - t)\gamma\bar{\lambda} W^{-\beta} \alpha^{\frac{1-2\beta}{2}} \left(\prod_{i=1}^l c_i\mu_{t,i}L_i + L_ic_i(t^{\prime} - t)\gamma\bar{\lambda}\mu_{t,i}\right)\left(\sum_{i=1}^l L_ic_i\mu_{t,i}\right)\text{.}$$
\end{proof}
\end{document}